\documentclass{article}
\usepackage{iclr2027_conference,times}

\usepackage[utf8]{inputenc}
\usepackage[T1]{fontenc}
\usepackage{amsmath,amssymb,amsfonts,amsthm}
\usepackage{graphicx}
\usepackage{booktabs,multirow,makecell,colortbl,array,threeparttable}
\usepackage{wrapfig,enumitem,appendix,float}
\usepackage{microtype,nicefrac,pifont,xspace}
\usepackage{xcolor,tcolorbox,tikz}
\usepackage{hyperref,url}
\hypersetup{hidelinks}

\newsavebox{\forgeTableBox}
\newcommand{\forgeTableFit}[2]{%
  \sbox{\forgeTableBox}{#2}%
  \ifdim\wd\forgeTableBox>#1\relax
    \resizebox{#1}{!}{\usebox{\forgeTableBox}}%
  \else\usebox{\forgeTableBox}\fi
}

\newcommand{\ours}{\textsc{Forge}\xspace}
\newtheorem{proposition}{Proposition}

\definecolor{lightblue}{RGB}{222,235,247}
\definecolor{lightgray}{RGB}{245,245,245}
\definecolor{natureblue}{RGB}{0,76,153}
\definecolor{naturepurple}{RGB}{115,65,130}
\definecolor{naturegray}{RGB}{248,248,248}
\definecolor{natureteal}{RGB}{32,128,128}
\definecolor{naturemagenta}{RGB}{170,40,120}
\definecolor{wine}{HTML}{830E0D}

\newcommand{\fancynumber}[1]{%
  \raisebox{1pt}{\tikz[baseline=(char.base)]{\node[
    shape=circle,draw=black,fill=natureblue!20,inner sep=0pt,
    minimum size=0.8em,font=\tiny,text=black](char){#1};}}%
}
\newtcolorbox{definitionbox}{
  colback=naturegray,colframe=natureblue!40,boxrule=1pt,arc=8pt,
  left=4pt,right=4pt,top=0pt,bottom=0pt,
}
\newtcolorbox{takeawaybox}[1][]{
  colback=naturegray,colframe=natureblue!40,title={\textsf{#1}},
  coltitle=black,boxrule=1pt,arc=8pt,
  left=4pt,right=4pt,top=0pt,bottom=0pt,
}
\newtcolorbox{hypothesisbox}{
  colback=naturegray,colframe=naturepurple!40,boxrule=1pt,arc=8pt,
  left=4pt,right=4pt,top=0pt,bottom=0pt,
}

\title{\ours: Form-Optimal Routing of Grounded Evidence for Frozen LLM Agents}

\def\forgeWebsiteURL{https://xixiaouab.github.io/projects/FORGE/}
\def\forgeCodeURL{https://github.com/xixiaouab/FORGE-code}
\newcommand{\forgeHeaderMark}{{\normalfont\small Preprint.}}
\newcommand{\forgeVersionDate}{2026-09-27}
\newcommand{\forgeHeaderDate}{{\normalfont\small\itshape\forgeVersionDate}}
\renewcommand{\headrulewidth}{0.4pt}

\makeatletter
\long\def\@author{%
  {\bfseries
  \mbox{Xi Xiao$^{1}$},
  \mbox{Yunbei Zhang$^{2}$},
  \mbox{Chen Liu$^{3}$},
  \mbox{Lin Zhao$^{4}$},
  \mbox{Jialin Chen$^{3}$},\\
  \mbox{Tianchen Zhao$^{5}$},
  \mbox{Xiang Xu$^{5}$},
  \mbox{Youngeun Kim$^{6}$},
  \mbox{Tianyang Wang$^{1}$},
  \mbox{Min Xu$^{7}$}\par}
  \vspace{4pt}
  {\normalfont\fontsize{10}{12.5}\selectfont
  \mbox{$^{1}$University of Alabama at Birmingham}\quad
  \mbox{$^{2}$Tulane University}\\
  \mbox{$^{3}$Yale University}\quad
  \mbox{$^{4}$Northeastern University}\quad
  \mbox{$^{5}$Amazon AGI}\\
  \mbox{$^{6}$Korea University}\quad
  \mbox{$^{7}$Carnegie Mellon University}\par}
  \vspace{3pt}
  {\normalfont\fontsize{10}{12.5}\selectfont
  \ding{41}\hspace{0.3em}\href{mailto:xxiao@uab.edu}{\texttt{xxiao@uab.edu}}\par}
  \vspace{5pt}
  {\normalfont\fontsize{10}{12.5}\selectfont
  \raisebox{-0.4ex}{\includegraphics[height=1em]{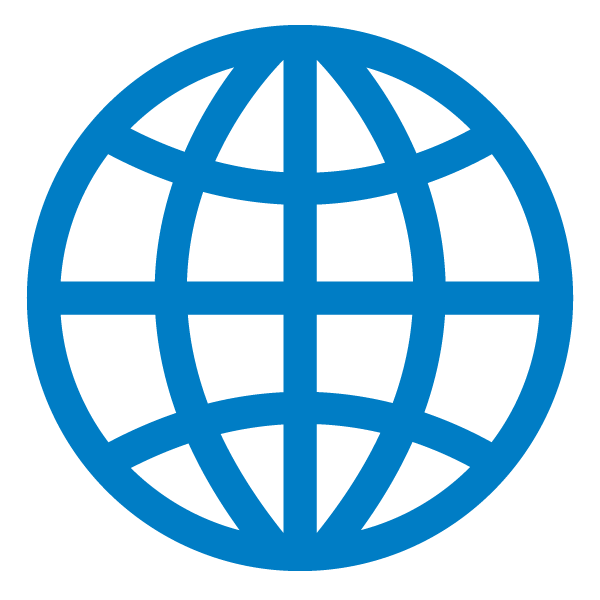}}\hspace{0.3em}%
  \ifx\forgeWebsiteURL\empty\texttt{Website}\else
    \href{\forgeWebsiteURL}{\texttt{Website}}\fi
  \hspace{1em}%
  \raisebox{-0.4ex}{\includegraphics[height=1em]{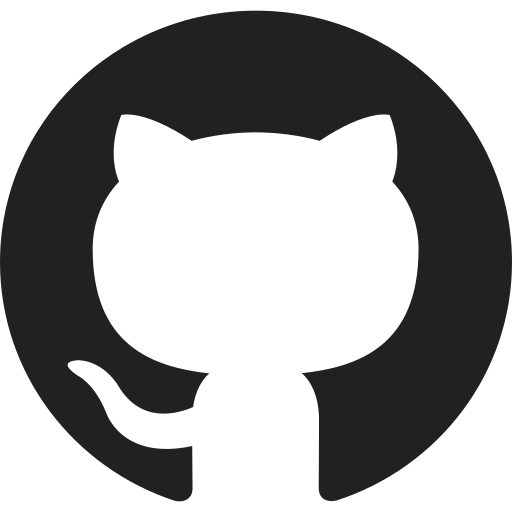}}\hspace{0.3em}%
  \ifx\forgeCodeURL\empty\texttt{Code}\else
    \href{\forgeCodeURL}{\texttt{Code}}\fi
  \par}%
}

\makeatother

\makeatletter
\renewcommand{\@maketitle}{%
  \vbox{\hsize\textwidth
    {\raggedright\hyphenpenalty=10000\exhyphenpenalty=10000
     \fontsize{18}{21.5}\selectfont\bfseries\@title\par}
    \vskip 13pt
    {\raggedright\fontsize{10}{13}\selectfont\@author\par}
    \vskip 18pt
  }%
}
\makeatother

\iclrfinalcopy %
\begin{document}
\raggedbottom
\maketitle
\fancyhead[L]{\forgeHeaderMark}
\fancyhead[R]{\forgeHeaderDate}
\vspace{-6pt}

\begin{abstract}
In agentic AI systems, frozen foundation models are increasingly deployed as closed-weight API endpoints, making downstream adaptation possible only through the inputs and inference procedures surrounding the model. As a result, for each input query, two coupled decisions largely determine both answer quality and token cost: \emph{what evidence to provide} and \emph{how much reasoning budget to allocate}. Fixed defaults along these axes are often suboptimal, misallocating support form or reasoning depth on roughly $80\%$ of queries in our analysis. To address this challenge, we propose \ours, a unified framework for adapting frozen models through per-query routing over a joint action space that spans both support form and thinking depth. Under an entropy-regularized, cost-aware utility objective, we derive a closed-form Boltzmann routing target and instantiate the policy as a lightweight 269K-parameter factorized router. The routing policy is trained around the frozen host, without any weight access, through a three-stage pipeline: offline arm enumeration, supervised Kullback-Leibler (KL) distillation from the Boltzmann target, and Group Relative Policy Optimization (GRPO) refinement with host feedback. Across 5 knowledge-intensive benchmarks and 8 frozen backbones ranging from 7B to 671B parameters, \ours improves accuracy at 42--45\% lower token cost on both main hosts, transfers zero-shot across hosts at lower token cost, and composes with intrinsic thinking budgets where available.
\end{abstract}

\begin{figure}[H]
  \centering
  \includegraphics[width=0.95\linewidth]{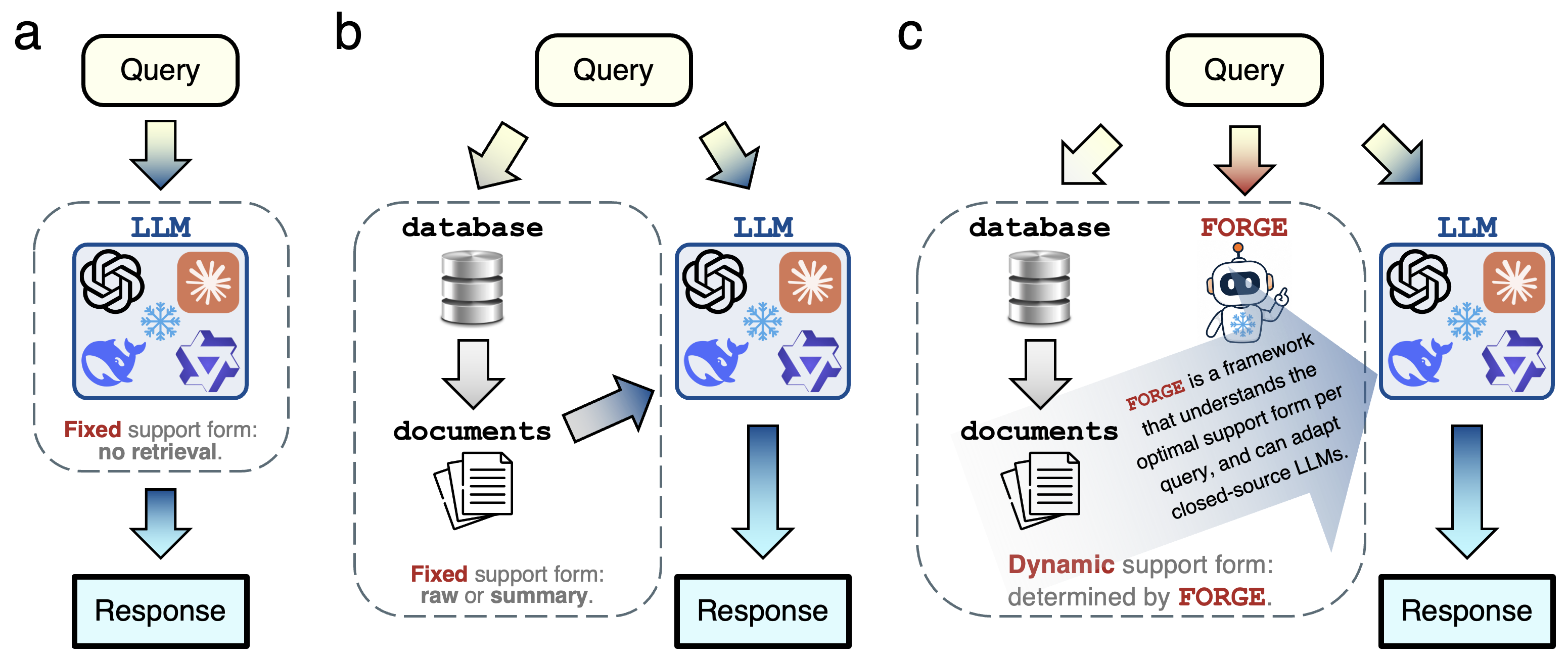}
  \caption{Fixed support forms in frozen-host workflows (a,b) and per-query
  support-form selection with \ours (c).}
  \label{fig:teaser}
\end{figure}
\clearpage

\section{Introduction}
\label{sec:intro}

Agentic AI has seen remarkable progress in recent years, with many of the most capable systems built on top of large foundation models~\citep{achiam2023gpt, comanici2025gemini, sapkota2025ai, ferrag2025llm}. Aligning these general-purpose models with task-specific requirements usually relies on adaptation via fine-tuning, using methods such as low-rank adaptation~(LoRA)~\citep{hu2022lora} and residual adapters~\citep{rebuffi2017learning}. However, foundation models are increasingly deployed as closed API endpoints, with private weights and no access to internal states, making host-side fine-tuning infeasible~\citep{sun2022black, cheng2024black, gao2023retrieval, asawa2025advisor}. This raises a central question: how can we effectively adapt frozen, closed-weight foundation models to downstream tasks?

For a fixed frozen host, effective adaptation hinges on two axes: \emph{what evidence to provide}~\citep{jeong2024adaptiverag, asai2024selfrag, zhu2026tiermem, lewis2020retrieval, karpukhin2020dense},
and \emph{how much reasoning budget to allocate}~\citep{openai2024o1, deepseek2025r1, anthropic2025claude37, wang2025adareasoner, wei2022chain, snell2024scaling, press2023measuring, trivedi2023interleaving, yue2025don,muennighoff2025s1}. In practice, however, agentic AI workflows commonly rely on fixed forms of support applied uniformly across queries, such as direct inference or retrieval-augmented generation with a predefined evidence format (Figure~\ref{fig:teaser}). This uniform treatment can be suboptimal because the most effective support form and reasoning budget vary substantially across queries. A controlled probe of a frozen \texttt{Qwen3-8B} on \texttt{HotpotQA} ($N{=}2{,}500$) makes this gap concrete (Figure~\ref{fig:support_forms}). On the evidence axis, $7.3\%$ of queries are answered \emph{more} accurately without retrieval than with full passages, $79.4\%$ achieve identical F1 under a query-aware compressed summary as under the full top-$k$, and only the remaining $20.6\%$ show any F1 difference between the two. On the reasoning axis with \texttt{Qwen3-8B}, $12\%$ of queries are \emph{harmed} by long reasoning, $60\%$ reach identical F1 under \textsc{NoThink} as under \textsc{Think-High}, and only $28\%$ genuinely benefit. These results suggest that both axes call for adaptive, per-query decisions. Beyond this, the two decisions are inherently coupled: better evidence reduces the reasoning needed to reach an answer, and deeper reasoning reduces the evidence required to support it. Addressing either alone leaves this coupling unexploited.

This observation raises three open questions.
\fancynumber{1}~\textbf{How should the two axes be jointly adapted at the query level?} Adaptive-retrieval and adaptive-reasoning methods each address one axis while fixing the other, leaving no formulation that adapts both together per query. 
\fancynumber{2}~\textbf{How can such a joint policy be trained around a frozen, weight-inaccessible host?} Since the host cannot be finetuned, adaptation must rely on a separate router that learns from external signals such as offline labels and online host feedback. However, existing methods treat these two signals as separate training paradigms, rather than as complementary signals for a common cost-aware objective. 
\fancynumber{3}~\textbf{What target should such a policy be trained toward?} Existing routers define training targets heuristically, without a closed-form connection to the underlying cost-aware decision problem.

\begin{figure}[!b]
  \centering
  \includegraphics[width=\linewidth]{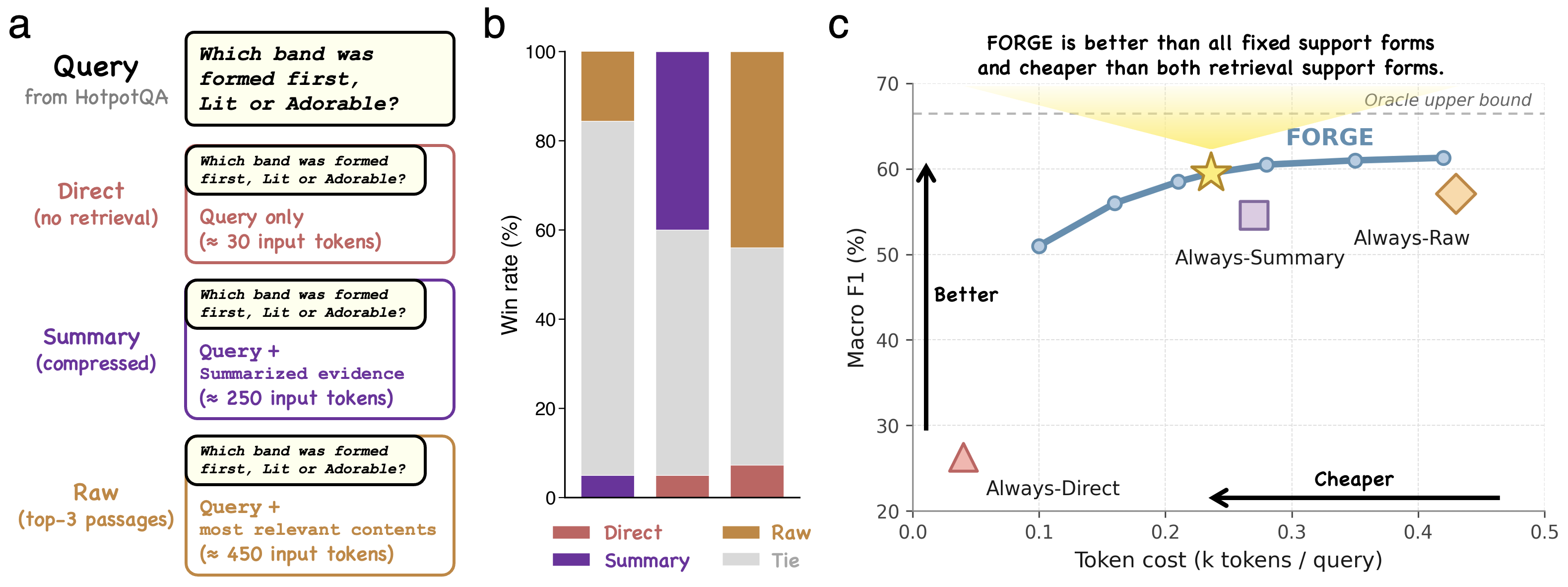}
  \caption{\textbf{Frozen LLM agents are picky readers.} \textbf{a.} The three common support forms. \textbf{b.} No support form is universally optimal; the best choice depends on the input query. \textbf{c.} Our proposed \ours establishes a new Pareto frontier over fixed support forms.}
  \label{fig:support_forms}
\end{figure}

\emph{\textbf{To answer~\fancynumber{1}, we formalize joint adaptation as per-query routing over a unified action space.}}
The action $a = (\mu, \theta)$ pairs a \emph{support form} $\mu \in \{\mathrm{Direct}, \mathrm{Summary}, \mathrm{Raw}\}$, corresponding to three evidence-channel rates between the corpus and the host (no retrieval, query-aware compressed retrieval, and full top-$k$ passages), with a \emph{thinking depth} $\theta$ (two settings on non-thinking hosts and four on hosts that expose test-time compute control).
For each query, the router 
selects a single joint action $a$ that maximizes the Pareto utility 
$U(x, a) = F_1(x, a) - \lambda_\mathrm{in}\tilde c_\mathrm{in}(a) - \lambda_\mathrm{out}\tilde c_\mathrm{out}(x, a)$, 
trading answer F1 against normalized input and output token cost. 
Crucially, the two axes are no longer optimized in isolation: a single 
utility evaluates the joint choice and a single policy outputs it per query.
Within this formulation, \ours instantiates the 
joint policy as a 269K-parameter factorized router that selects the action in a single 
forward pass before the frozen host answers. Importantly, standard RAG, Adaptive-RAG, and tiered-memory routing arise as degenerate special cases of the proposed \ours by fixing one or more routing dimensions.

\emph{\textbf{To answer~\fancynumber{2}, we train the router through a three-stage pipeline that runs entirely around the frozen host without changing its weights.}}
Stage 0 enumerates a 3-arm warm-start subset $\mathcal{A}_\mathrm{warm} = \mathcal{M} \times \{\mathrm{NoThink}\}$ on $1{,}500$ training queries, recording per-arm F1 and token cost ($\approx 4{,}500$ host calls, under one hour).
Stage 1 fits the warm-start router $\pi_{\theta_0}$ to a Boltzmann target restricted to $\mathcal{A}_\mathrm{warm}$ via KL distillation (CPU-only, two minutes).
Stage 2 lifts the policy to the full alphabet $\mathcal{A}$ via GRPO~\citep{shao2024deepseekmath} with sampled host rewards and a KL anchor to $\pi_{\theta_0}$ ($8$--$12$ hours on one 8-GPU node).
The same pipeline produces the $6$-arm router for non-thinking hosts, the $12$-arm joint-policy router that composes with the reasoning-budget axis on capable hosts, and a single router that transfers zero-shot across hosts. Router computation adds $0.5$--$5.5$ ms per query; on a fresh query, the full variant also issues four parallel host probes for its features, whereas \ours-Lite needs only the final answer call (\S\ref{sec:main_results}, App.~\ref{app:latency}).

\emph{\textbf{To answer~\fancynumber{3}, we ground the router in a unified Pareto reward and a closed-form Boltzmann target.}}
The Pareto utility above is the system's single source of truth: every training stage, every baseline, and every Pareto operating point sweeps along it.
Adding an entropy regularizer to the cost-aware utility maximization, as in maximum-entropy reinforcement learning~\citep{ziebart2008maximum, haarnoja2018soft}, fixes a unique closed form on the discrete action space, $\pi^*(a \mid x) \propto \exp(U(x, a) / \tau)$. The same softmax form is also the choice rule of discrete-choice random-utility models~\citep{mcfadden1974conditional}, suggesting that the Boltzmann target is a natural cost-aware decision rule rather than an ad hoc soft label. Stage 1 distills this target on the warm-start actions, and Stage 2 refines the policy against the same utility with sampled host rewards; App.~\ref{app:regret} characterizes the idealized fixed point of the KL-anchored Stage-2 objective.
Proposition~\ref{prop:warm_start_factorization} (App.~\ref{app:warmstart_proof}) further decomposes the warm-start mismatch into a support-form marginal gap plus a closed-form $\log|\Theta|{-}\mathcal{H}$ cost from the uniform thinking head. These results characterize the target and its initialization rather than guarantee convergence of the implemented optimizer.

Across $5$ knowledge-intensive benchmarks and $8$ frozen LLM backbones from 7B to 671B parameters across three deployment regimes (non-thinking local, thinking-capable, frontier non-thinking), \ours Pareto-dominates Always-Raw on every cell tested on \texttt{Qwen3-8B} and \texttt{Mistral-7B} ($+2.4$ and $+2.9$ macro F1 at $45\%$ and $42\%$ lower cached token cost; \S\ref{sec:main_results}).
A single router trained on \texttt{Qwen3-8B} transfers zero-shot to three frontier backends, matching or exceeding Always-Raw F1 on $11$ of $15$ task / host cells at $30\%$+ lower cached cost (\S\ref{sec:transfer}).
On thinking-capable hosts the joint $12$-arm policy exceeds Always-Raw with maximum thinking budget by $+1.4$ to $+1.9$ macro F1 at roughly half the total token cost (\S\ref{sec:joint_thinking}). We
believe that this work contributes to the tractable, theoretically grounded, and broadly transferable axis of frozen-agent adaptation.
\section{Preliminaries and Related Works}
\label{sec:prelim}
\subsection{Related Works}
\label{sec:related}

\emph{Model routing} selects the host~\citep{chen2024frugalgpt, ong2024routellm, jiang2025s3, cheng2024xrag, su2024dragin, jiang2023active, hu2024routerbench, dinghybrid}; \emph{adaptive retrieval} selects whether and what to retrieve~\citep{jeong2024adaptiverag, asai2024selfrag, zhu2026tiermem}; \emph{thinking-budget control} selects reasoning effort~\citep{wang2025adareasoner, alomrani2025reasoningbudget, welleckdecoding, yao2023tree}.
Building on these choices, including partial support-form selection, \ours jointly selects Direct, Summary, or Raw evidence and a host-exposed thinking setting under one measured quality--cost utility (extended discussion in App.~\ref{app:related}).

\subsection{Routing Objective and Boltzmann Solution}
\label{sec:formulation}
\label{sec:theory}
\label{sec:theory:maxent}

Let $A$ be a frozen pre-trained LLM with fixed, inaccessible parameters, $x$ a query, and $\mathcal{D}$ a document corpus.
The router selects $a = (\mu, \theta) \in \mathcal{A} = \mathcal{M} \times \Theta$, with support $\mu \in \{\mathrm{Direct}, \mathrm{Summary}, \mathrm{Raw}\}$ and thinking configuration $\theta$ (\S\ref{sec:arms}). Using prompt constructor $g_a$, the host returns $\hat y = A(g_a(x, \mathcal{D}))$, with F1 $m(x, a) = \mathrm{F1}(\hat y, y)$ and input/output token costs $c_\mathrm{in}(a), c_\mathrm{out}(x, a)$. Dataset-level maxima normalize these costs to $\tilde c_\mathrm{in}, \tilde c_\mathrm{out}$, yielding the per-query \emph{Pareto utility}:
\begin{equation}
\label{eq:pareto_utility}
U(x, a) \;=\; m(x, a) \;-\; \lambda_\mathrm{in} \tilde c_\mathrm{in}(a) \;-\; \lambda_\mathrm{out} \tilde c_\mathrm{out}(x, a),
\quad \lambda_\mathrm{in}, \lambda_\mathrm{out} \geq 0,
\end{equation}
The weights set the quality--cost operating point; $\lambda_\mathrm{out}{=}0$ penalizes input cost only. Labeled queries provide $m(x,a)$ for executed actions during training and evaluation. At inference, routing uses pre-answer features without observing the selected response's F1.

For a fixed query $x$, finite action set $\mathcal{A}$, finite utilities $U(x,a)$, and temperature $\tau>0$, the entropy-regularized objective $\max_{\pi(\cdot\mid x)\in\Delta(\mathcal{A})}\sum_a\pi(a\mid x)U(x,a)+\tau H(\pi(\cdot\mid x))$ has the unique solution
\begin{equation}
\label{eq:boltzmann}
\pi^*(a \mid x) \;=\; \frac{\exp\!\big(U(x, a)/\tau\big)}{\sum_{a'} \exp\!\big(U(x, a')/\tau\big)},
\end{equation}
As $\tau\to0$, it concentrates uniformly on utility maximizers; as $\tau\to\infty$, it becomes uniform. The standard solution gives Stage~1 soft labels on enumerated warm-start actions. Optimality holds for the stated entropy-regularized objective, without guaranteeing that the parameterized router or clipped Stage~2 update attains it. Stage~2 refines the policy with sampled host rewards over the full combination of support forms and thinking settings (\S\ref{sec:training}).

\subsection{The Adaptive-Thinking Action Alphabet}
\label{sec:arms}

The support form $\mu$ is \textit{Direct}, the query alone (negligible input tokens); \textit{Summary}, BM25 retrieval with query-aware sentence-level compression ($\approx 250$ input tokens); or \textit{Raw}, full top-$k$ BM25 passages ($\approx 450$ input tokens).
Thinking depth $\theta$ is \textit{NoThink} (plain decoding), \textit{CoT-Prompt} (the in-context ``Let us think step by step'' suffix), or, when the host exposes a reasoning budget, \textit{Think-Low} or \textit{Think-High} (construction in App.~\ref{appendix:retrieval}).
Hosts without a reasoning mode use $\Theta_\mathrm{nt}{=}\{\mathrm{NoThink}, \mathrm{CoT\text{-}Prompt}\}$ ($K{=}6$); thinking-capable hosts use $\Theta_\mathrm{t}{=}\{\mathrm{NoThink}, \mathrm{CoT\text{-}Prompt}, \mathrm{Think\text{-}Low}, \mathrm{Think\text{-}High}\}$ ($K{=}12$).

Query-dependent utility maximizers motivate adaptive routing. Proposition~\ref{prop:routing} (App.~\ref{app:background_props}) specifies when a per-query utility oracle strictly exceeds every fixed action in expectation; \S\ref{sec:exp} and \S\ref{sec:transfer} examine action heterogeneity and six-arm references. The \emph{Oracle} rows are descriptive references: their selection rule was not retained, precluding claims of utility optimality or F1 upper bounds.

\section{Method}
\label{sec:method}

\ours jointly controls evidence delivery and reasoning depth before the frozen host answers. The insight is their interdependence: evidence quality changes the value of extra reasoning, while the reasoning setting changes what support is worth providing. Candidate actions' F1 and input/output tokens define the utility in Eq.~\eqref{eq:pareto_utility} and soft targets in Eq.~\eqref{eq:boltzmann}. Figures~\ref{fig:architecture} and~\ref{fig:router} summarize training and routing.

\subsection{Router and Deployment Features}
\label{sec:router}
The policy factors into a support-form head and a support-conditioned reasoning head:
\begin{equation}
\label{eq:factorized}
\pi_\theta(a\mid x)=\pi_\theta^{\mu}(\mu\mid h(x))\,\pi_\theta^{\theta}(\theta\mid h(x),\mu).
\end{equation}
A shared two-layer MLP maps inputs through 256 hidden units to a 256-dimensional representation (ReLU, dropout 0.1); the heads use a learned 16-dimensional support-form embedding. The six-arm Full router has approximately 269K parameters (settings in App.~\ref{appendix:router}). Full \ours uses a frozen 768-dimensional BGE embedding and 21 structured features (789 total). Eleven features depend on one greedy host probe and three self-consistency samples, requiring \emph{four host calls before} the routed answer on a fresh query. Caching shifts these calls outside routed-answer accounting but does not eliminate their fresh-query cost. \textsc{Forge}-Lite is retrained on 778 host-independent features: the embedding, four BM25 statistics, one query--top-passage cosine similarity, and five structural features. It requires no training or inference probes and only one final-answer host call. Both variants share the single-step action set and frozen retrieval pipeline. \emph{Cached} counts routed-answer tokens; \emph{Fresh Online} tokens and latency include every host call for a new query (Appendix Figure~\ref{fig:router_deployment}).

\subsection{Training}
\label{sec:training}

\begin{wrapfigure}[22]{R}{0.43\linewidth}
\centering
\includegraphics[width=\linewidth]{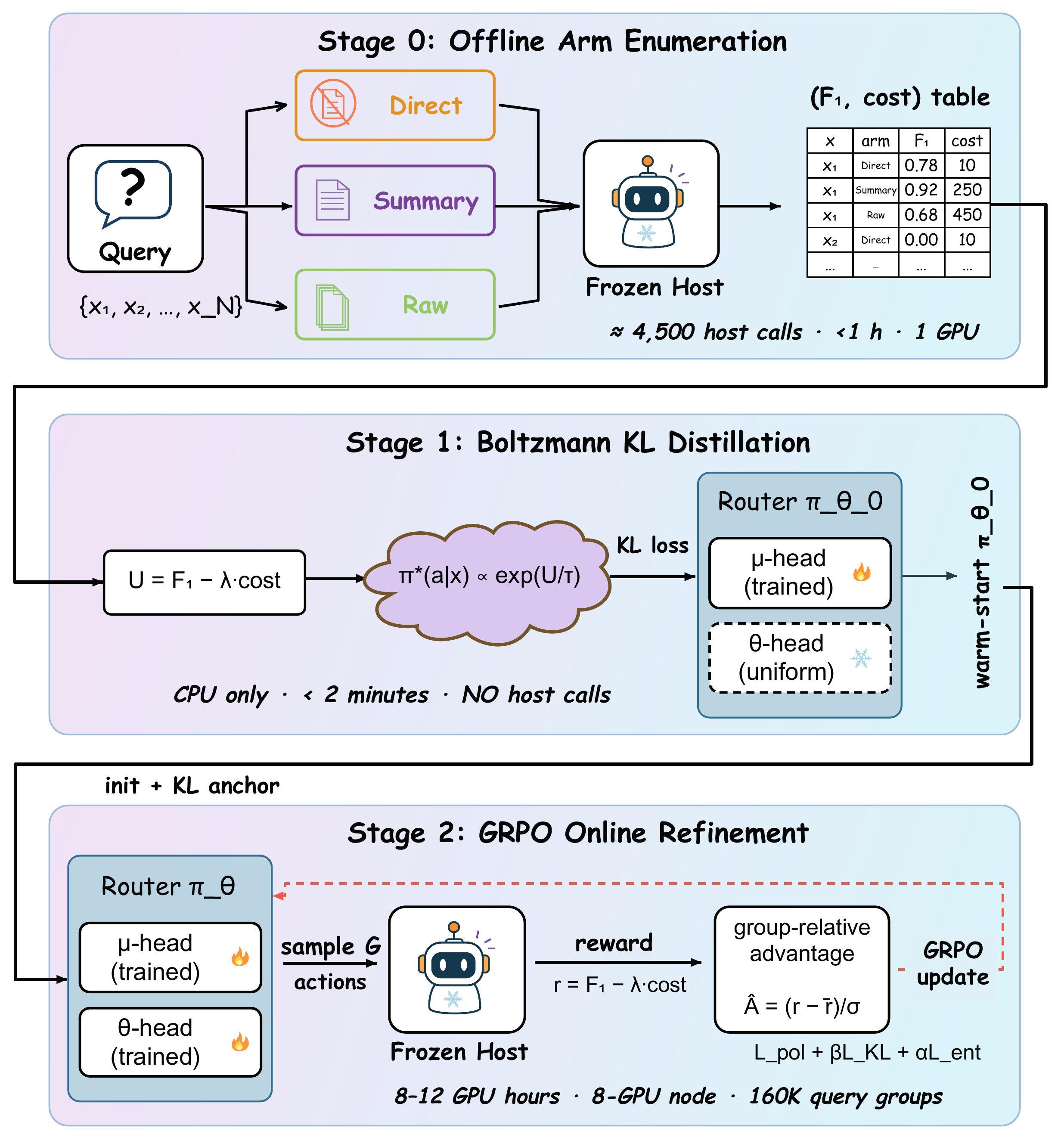}
\caption{\textbf{Three-stage training around a frozen host.} Stage~2 uses $5{,}000\times32=160$K query groups with eight sampled actions each: 1.28M host completions.}
\label{fig:architecture}
\end{wrapfigure}
\paragraph{Stage 0: action evaluation.}
On each labeled warm-start query $x_i$, the frozen host executes $a\in\mathcal A_{\mathrm{warm}}=\mathcal M\times\{\mathrm{NoThink}\}$ and records $(m_i^a,c_{\mathrm{in},i}^a,c_{\mathrm{out},i}^a)$. These calls incur training cost without updating host weights. Gold answers provide $m_i^a$; \S\ref{sec:exp} tests proxy rewards with gold-labeled development data for router hyperparameter selection.

\paragraph{Stage 1: supervised distillation.}
Using $U(x_i,a)$ from Eq.~\eqref{eq:pareto_utility}, the evidence head minimizes $D_{\mathrm{KL}}(\pi^*_{\mathrm{Boltz}}\Vert\pi_\theta^\mu)$ against Eq.~\eqref{eq:boltzmann} on the three warm-start actions at $\tau_0=1$; the reasoning head starts uniform. A separate Stage-1-only control uses six-arm offline observations, holding the full action set fixed to isolate the contribution of Stage~2.

\paragraph{Stage 2: policy-guided refinement.}
Over the full action set, each of $B$ queries receives $G$ sampled actions, stratified to cover every evidence form when $G\geq |\mathcal M|$. Frozen-host executions yield $r_{b,g}=U(x_b,a_{b,g})$ and normalized group-relative advantages
\begin{equation}
\label{eq:advantage}
\hat A_{b,g}=\frac{r_{b,g}-\bar r_b}{\sigma_{r_b}+\epsilon},\qquad
\bar r_b=\tfrac{1}{G}\sum_g r_{b,g},
\end{equation}
where $\epsilon=10^{-4}$ and zero-variance groups receive zero advantage. A PPO-clipped policy term~\citep{schulman2017ppo,shao2024deepseekmath} uses this feedback, with a KL anchor to Stage~1 and an entropy bonus. Appendix~\ref{appendix:router} specifies sampling and the surrogate; the same-action-set ablation (\S\ref{sec:exp}) isolates refinement from alphabet expansion.

\begin{figure}[!t]
\centering
\includegraphics[width=\linewidth]{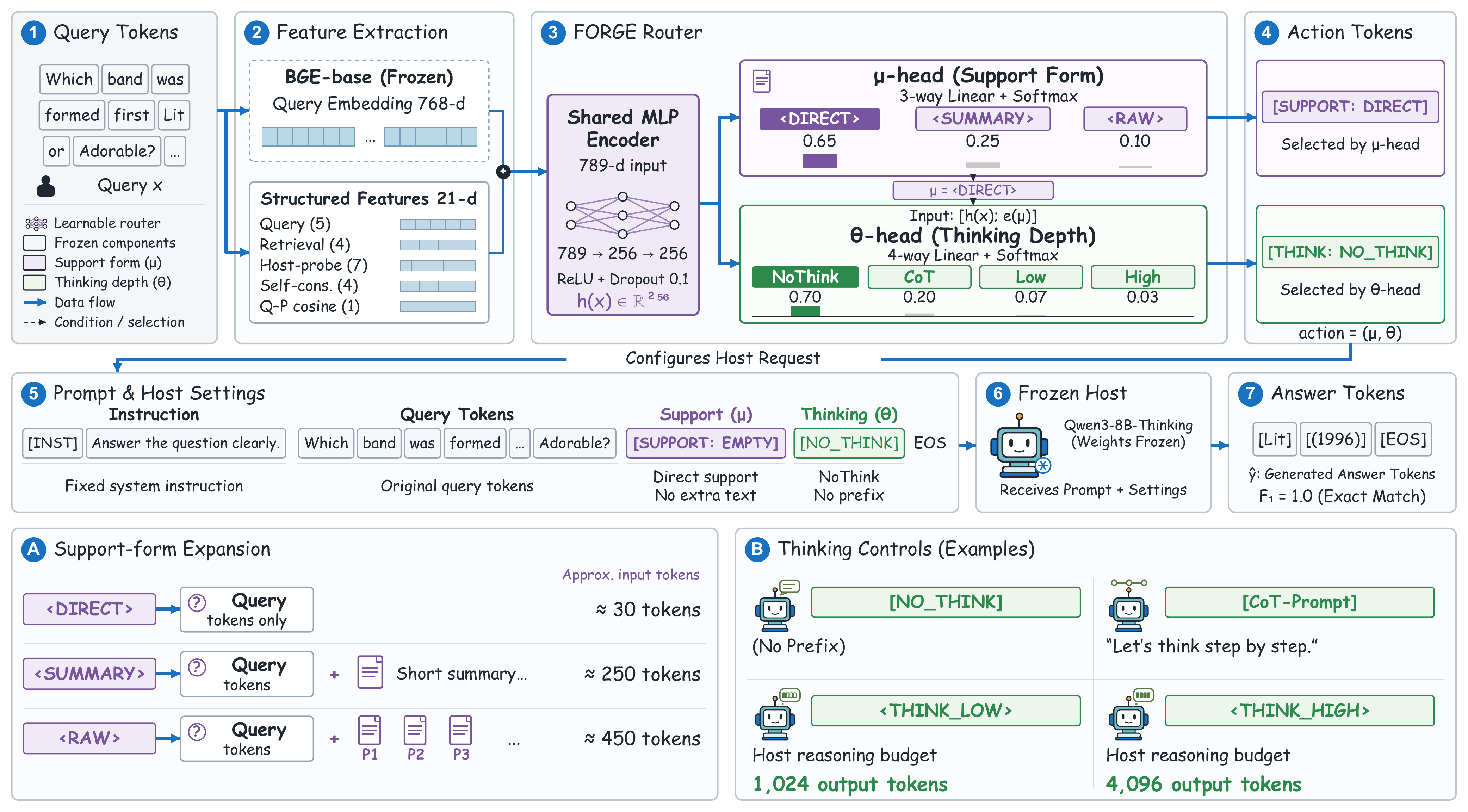}
\caption{\textbf{Factorized router and frozen-host inference.} The $789\to256\to256$ encoder feeds support and support-conditioned thinking heads (features: Table~\ref{tab:features}). Bottom panels expand support choices and alternative settings of one host: CoT is prompted; Low/High use host-exposed budgets. Probabilities, token counts, and the answer are illustrative.}
\label{fig:router}
\end{figure}

\section{Experiments}
\label{sec:exp}

\subsection{Setup and Evaluation Protocols}
\label{sec:setup}

We evaluate HotpotQA~\citep{yang2018hotpotqa}, 2WikiMultiHopQA~\citep{ho2020constructing}, MuSiQue~\citep{trivedi2022musique}, PopQA~\citep{mallen2023popqa}, and FEVER~\citep{thorne2018fever} on Qwen3-8B-Instruct and Mistral-7B-Instruct-v0.3, plus three thinking-capable and three frontier hosts. The six-arm alphabet pairs Direct/Summary/Raw with NoThink/CoT-Prompt; thinking hosts add Think-Low/High for 12 arms. BM25 retrieval is fixed; Summary deterministically extracts from the top three passages. Splits and retrieval settings are in Appendices~\ref{appendix:datasets} and~\ref{appendix:retrieval}.

We report token-level F1, exact match (EM), and cost $\bar C=\bar C_{\rm in}+\bar C_{\rm out}$ in k host tokens/query. \emph{Cached} cost counts the routed answer with precomputed features; \emph{Fresh Online} includes feature acquisition and the answer on a new query: four pre-route host calls for Full, none for Lite. Baselines are fixed Direct/Summary/Raw, BM25-Threshold, Adaptive-RAG~\citep{jeong2024adaptiverag}, TierMem~\citep{zhu2026tiermem}, s3~\citep{jiang2025s3}, and Sysformer~\citep{sysformer2025}. The matched BGE+BM25-KL router uses 768 BGE dimensions, four BM25 statistics, and one query--passage cosine, with the same Stage-0 data, MLP capacity, optimizer, and development budget as \ours, while excluding GRPO refinement and host-dependent feature probes.

\subsection{Answer Quality and the Cost of a New Query}
\label{sec:main_results}

\begin{table}[!htbp]
\caption{\textbf{Canonical Cached results.} Task F1 (\%) and selected-answer cost $\bar C$ (k/query); HQA/MSQ/Pop/FEV denote HotpotQA/MuSiQue/PopQA/FEVER. Bold/boxed values mark best practical task/macro F1 per host. Oracle is a descriptive six-arm reference; Table~\ref{tab:online} gives online costs, including the feature calls needed before answering a new query.}
\label{tab:main}
\centering
\small
\renewcommand{\arraystretch}{1.0}
\setlength{\tabcolsep}{1.8pt}
\begin{tabular}{@{}lrrrrrrr@{\hspace{5.5pt}}rrrrrrr@{}}
\toprule
\multirow{2}{*}{Policy} & \multicolumn{7}{c}{Qwen3-8B} & \multicolumn{7}{c}{Mistral-7B}\\
\cmidrule(lr){2-8}\cmidrule(l){9-15}
& HQA & 2Wiki & MSQ & Pop & FEV & Macro & $\bar C$ & HQA & 2Wiki & MSQ & Pop & FEV & Macro & $\bar C$\\
\midrule
Always-Direct & 26.2 & 25.7 & 10.2 & 15.0 & 54.4 & 26.3 & 0.040 & 25.2 & 17.7 & 7.9 & 23.0 & 47.0 & 24.2 & 0.041\\
Always-Summary & 46.9 & 38.1 & 19.0 & 87.0 & \textbf{81.8} & 54.6 & 0.270 & 38.3 & 24.5 & 11.8 & 80.0 & 69.5 & 44.8 & 0.265\\
Always-Raw & 57.2 & 40.5 & 24.1 & 83.9 & 79.6 & 57.1 & 0.430 & 46.6 & 28.7 & 14.0 & 74.8 & 72.0 & 47.2 & 0.428\\
\cmidrule(lr){1-8}\cmidrule(l){9-15}
BM25-Threshold & 50.5 & 38.0 & 19.5 & 80.5 & 78.0 & 53.3 & 0.268 & 41.5 & 26.2 & 12.5 & 73.0 & 69.0 & 44.4 & 0.263\\
Adaptive-RAG & 56.5 & 40.0 & 23.0 & 80.0 & 77.5 & 55.4 & 0.305 & 45.5 & 27.0 & 13.5 & 73.0 & 70.0 & 45.7 & 0.300\\
TierMem-2arm & 47.6 & 39.6 & 20.8 & 81.8 & \textbf{81.8} & 54.3 & 0.316 & 40.9 & 28.2 & 12.9 & 70.4 & 70.4 & 44.6 & 0.314\\
s3 & 55.0 & 39.0 & 21.5 & 78.5 & 75.5 & 53.9 & 0.305 & 45.5 & 26.5 & 13.8 & 71.0 & 67.0 & 44.8 & 0.310\\
Sysformer & 57.5 & 38.8 & 21.2 & 86.0 & 76.0 & 55.9 & 0.285 & 47.8 & 23.5 & 14.2 & 75.0 & 69.0 & 45.9 & 0.300\\
\cmidrule(lr){1-8}\cmidrule(l){9-15}
Stage 1 (3-arm) & 58.3 & 38.3 & 20.8 & 87.0 & 75.7 & 56.0 & 0.272 & 48.4 & 23.0 & 14.5 & 76.6 & 70.0 & 46.5 & 0.288\\
Stage 1 (6-arm) & 59.0 & 40.5 & 23.5 & 87.2 & 78.0 & 57.6 & 0.255 & 49.0 & 28.5 & 15.5 & 77.5 & 71.0 & 48.3 & 0.272\\
\rowcolor{lightblue!60}\ours & \textbf{60.4} & \textbf{42.7} & \textbf{26.4} & \textbf{87.5} & 80.5 & \textbf{\fbox{59.5}} & 0.236 & \textbf{50.2} & \textbf{30.5} & \textbf{16.7} & \textbf{80.5} & \textbf{72.5} & \textbf{\fbox{50.1}} & 0.249\\
Oracle (6-arm) & 67.4 & 53.2 & 34.0 & 88.5 & 89.2 & 66.5 & 0.168 & 59.1 & 42.6 & 23.2 & 83.4 & 82.0 & 58.1 & 0.182\\
\bottomrule
\end{tabular}
\end{table}

Full \ours reaches 59.5/50.1 macro F1 on Qwen/Mistral, exceeding Always-Raw by 2.4/2.9 points with 45\%/42\% fewer cached answer tokens (Table~\ref{tab:main}). It exceeds Sysformer, the strongest adapted baseline here, by 3.6/4.2 points; Table~\ref{tab:reliability} tests a matched lightweight comparator and isolates Stage~2 from action expansion. Fixed Summary leads on Qwen FEVER (81.8 vs.\ 80.5), which motivates the development-set fallback for selecting a task-specific operating point.

\begin{table}[!htbp]
\caption{\textbf{Fresh Online results (batch size 1).} All host calls and input/output tokens are counted; F1/EM are five-task macro scores. Latency is end-to-end, with Full's four feature calls parallelized. Bold marks best quality or lowest online cost per host.}
\label{tab:online}
\label{tab:fresh_online_app}
\centering
\small
\renewcommand{\arraystretch}{1.0}
\setlength{\tabcolsep}{6.5pt}
\begin{tabular}{@{}llrrrrrrr@{}}
\toprule
\multirow{2}{*}{Host} & \multirow{2}{*}{Policy} & \multicolumn{2}{c}{Host calls} & \multicolumn{2}{c}{Macro (\%)} & \multirow{2}{*}{\makecell{Online\\$\bar C$ (k)}} & \multicolumn{2}{c}{Latency (ms)}\\
\cmidrule(lr){3-4}\cmidrule(lr){5-6}\cmidrule(l){8-9}
& & Pre & Total & F1 & EM & & p50 & p95\\
\midrule
\multirow{3}{*}{Qwen3-8B}
& Always-Raw & 0 & 1 & 57.1 & 48.6 & 0.430 & 122.5 & 148.9\\
& \ours-Lite & 0 & 1 & 58.7 & 49.8 & \textbf{0.242} & 128.0 & 154.1\\
\rowcolor{lightblue!60}\cellcolor{white} & Full \ours & 4 & 5 & \textbf{59.4} & \textbf{50.5} & 0.384 & 271.4 & 329.8\\
\midrule
\multirow{3}{*}{Mistral-7B}
& Always-Raw & 0 & 1 & 47.2 & 40.5 & 0.428 & 129.8 & 157.7\\
& \ours-Lite & 0 & 1 & 49.1 & 41.9 & \textbf{0.256} & 135.3 & 164.8\\
\rowcolor{lightblue!60}\cellcolor{white} & Full \ours & 4 & 5 & \textbf{50.0} & \textbf{42.7} & 0.400 & 291.7 & 354.6\\
\bottomrule
\end{tabular}
\end{table}

Lite improves Qwen/Mistral F1 by 1.6/1.9 while saving 44\%/40\% of online tokens with roughly 5.5\,ms added p50 latency (Table~\ref{tab:online}). Full gains accuracy but raises Qwen p50 from 122.5 to 271.4\,ms. On the DeepSeek-V3.2 API, zero-shot Lite gives 63.1 F1/0.182k tokens/about 2.2\,s versus Raw's 63.4/0.287k/about 2.1\,s: 36.6\% fewer tokens for 0.3 F1 and a small latency increase. API limits prevented online Full measurement. Lite thus preserves one-call efficiency; Full offers additional accuracy when deployment permits the latency of parallel feature probes.

\subsection{Reliability, Strong Baselines, and Stage-2 Value}
\label{sec:ablation}
\label{sec:ablations}

\begin{table}[!htbp]
\caption{\textbf{Reliability and matched Stage~2 gains, Cached.} Mean $\pm$ SD: three splits $\times$ three seeds, equal baseline tuning budgets; 2,500-update rows are single sweep points. Paired-query bootstrap 95\% CIs use the canonical split, comparing policies with Raw or Full with six-arm Stage~1 at matched features, decoding, and cost weights. Boxes mark best nine-run means.}
\label{tab:reliability}
\label{tab:stage2_controlled_main}
\centering
\small
\renewcommand{\arraystretch}{1.0}
\setlength{\tabcolsep}{4pt}
\begin{tabular}{@{}llccccc@{}}
\toprule
\multirow{2}{*}{Host} & \multirow{2}{*}{Policy} & \multicolumn{2}{c}{Macro F1 evidence} & \multirow{2}{*}{\makecell{Cached\\$\bar C$ (k)}} & \multicolumn{2}{c}{Matched Stage~2 effect}\\
\cmidrule(lr){3-4}\cmidrule(l){6-7}
& & \makecell{Nine-run\\mean $\pm$ SD} & \makecell{$\Delta$ vs Raw\\95\% CI} & & Mean $\Delta$F1 & \makecell{Canonical\\95\% CI}\\
\midrule
\multirow{6}{*}{Qwen3-8B}
& Always-Raw & $57.0\pm0.4$ & reference & 0.431 & -- & --\\
& BGE+BM25-KL & $57.5\pm0.5$ & $[+0.1,+0.9]$ & 0.268 & -- & --\\
& \ours-Lite & $58.6\pm0.5$ & $[+1.1,+2.1]$ & 0.243 & -- & --\\
& Stage~1 (6-arm) & $57.6\pm0.5$ & -- & 0.255 & reference & --\\
& Stage~2 (2,500 updates) & 58.9 & -- & 0.244 & -- & --\\
\rowcolor{lightblue!60}\cellcolor{white} & Full \ours (5,000 updates) & \fbox{$\mathbf{59.4}$}$\pm0.4$ & $[+1.9,+2.9]$ & \textbf{0.237} & $+1.8$ & $[+1.3,+2.3]$\\
\midrule
\multirow{6}{*}{Mistral-7B}
& Always-Raw & $47.1\pm0.4$ & reference & 0.429 & -- & --\\
& BGE+BM25-KL & $47.5\pm0.5$ & $[0.0,+0.8]$ & 0.279 & -- & --\\
& \ours-Lite & $49.0\pm0.5$ & $[+1.4,+2.4]$ & 0.257 & -- & --\\
& Stage~1 (6-arm) & $48.3\pm0.5$ & -- & 0.272 & reference & --\\
& Stage~2 (2,500 updates) & 49.5 & -- & 0.259 & -- & --\\
\rowcolor{lightblue!60}\cellcolor{white} & Full \ours (5,000 updates) & \fbox{$\mathbf{50.0}$}$\pm0.5$ & $[+2.3,+3.5]$ & \textbf{0.250} & $+1.7$ & $[+1.2,+2.2]$\\
\bottomrule
\end{tabular}
\end{table}

\begin{wraptable}{R}{0.55\linewidth}
\caption{\textbf{Cumulative feature ablation.} Qwen3-8B Cached F1 (\%); parentheses show changes from the previous row. Full $\phi$ is the final row, evaluated in a separate feature-study run from Table~\ref{tab:main}.}
\label{tab:feature_ablation}
\centering
\small
\setlength{\tabcolsep}{5pt}
\renewcommand{\arraystretch}{1.07}
\begin{tabular}{@{}lcc@{}}
\toprule
Features & HotpotQA & MuSiQue\\
\midrule
Structured only & 56.8 & 23.4\\
$+$ BGE embedding & 60.4\,{\footnotesize$(+3.6)$} & 25.9\,{\footnotesize$(+2.5)$}\\
$+$ Retrieval & 59.3\,{\footnotesize$(-1.1)$} & 26.4\,{\footnotesize$(+0.5)$}\\
$+$ Host probe & \textbf{61.6}\,{\footnotesize$(+2.3)$} & 26.4\,{\footnotesize$(+0.0)$}\\
\rowcolor{lightblue!60}$+$ Self-consistency & 60.6\,{\footnotesize$(-1.0)$} & \textbf{26.9}\,{\footnotesize$(+0.5)$}\\
\bottomrule
\end{tabular}
\end{wraptable}

BGE+BM25-KL's Mistral interval starts at zero; the five retuned adaptive baselines have intervals below zero against Raw on both hosts. The canonical Qwen/Mistral ladder separates labels, actions, structure, and host features: hard-label BGE+BM25 56.8/46.7, KL distillation 57.5/47.5, six-arm routing with the same 773 features 58.4/48.6, Lite 58.7/49.1, and Full 59.4/50.0 macro F1.

Query embeddings yield the largest HotpotQA gain, while the host probe adds 2.3 F1 (Table~\ref{tab:feature_ablation}). Retrieval adds four BM25 statistics and one query--passage cosine. Appendix~\ref{appendix:features} gives the matched 773-to-789-dimensional comparison.

At matched actions, features, decoding, and cost weights, Stage~2 adds 1.8/1.7 macro F1 on Qwen/Mistral and saves 7.1\%/8.1\% cached tokens (Table~\ref{tab:stage2_controlled_main}); 2,500 updates deliver about 70\% of the final gain. Refinement thus improves decisions within the same action space. Training uses $5{,}000\times32\times8=1.28$M host completions, about 384M input/82M output tokens, and 64--96 GPU-hours on one eight-GPU node per operating point. This one-time training cost is separate from recurring feature probes, whose online overhead is measured in Table~\ref{tab:online}.

\subsection{Grounding, Proxy Rewards, and Operating Points}
\label{sec:grounding}

\begin{table}[!htbp]
\centering
\begin{minipage}[t]{0.485\linewidth}
\caption{\textbf{Grounding on Qwen3-8B (\%).} Canonical predictions; bold marks the best comparable value for each metric within each task.}
\label{tab:grounding_main}
\centering
\small
\setlength{\tabcolsep}{2.6pt}
\renewcommand{\arraystretch}{1.0}
\begin{tabular}{@{}lrrrrr@{}}
\toprule
Policy & F1 & \makecell{Gold\\recall} & \makecell{Source\\supp.} & \makecell{Prompt\\supp.} & Unsup.\\
\midrule
\multicolumn{6}{@{}l}{\emph{HotpotQA}}\\
Raw & 57.2 & \textbf{91.2} & 61.9 & 59.1 & 30.8 \\
Summary & 46.9 & 73.4 & 54.0 & 51.8 & 37.1 \\
\rowcolor{lightblue!60}Full & \textbf{60.4} & 64.8 & \textbf{65.3} & 63.7 & \textbf{27.4} \\
\cmidrule(l){1-6}
\multicolumn{6}{@{}l}{\emph{FEVER}}\\
Raw & 79.6 & \textbf{94.1} & 82.0 & 80.3 & 12.9 \\
Summary & \textbf{81.8} & 87.6 & \textbf{84.4} & 83.1 & \textbf{10.4} \\
\rowcolor{lightblue!60}Full & 80.5 & 79.8 & 83.7 & 82.5 & 10.9 \\
\bottomrule
\end{tabular}
\end{minipage}\hfill
\begin{minipage}[t]{0.485\linewidth}
\caption{\textbf{Frozen proxy rewards, Qwen3-8B.} Best outcomes are bold; boxes mark the highest macro F1 and exact match.}
\label{tab:proxy_reward_main}
\centering
\small
\setlength{\tabcolsep}{2.6pt}
\renewcommand{\arraystretch}{1.0}
\begin{tabular}{@{}llrrrr@{}}
\toprule
\makecell[l]{Stage 0/1\\reward} & \makecell[l]{Stage 2\\reward} & F1 & EM & $\bar C$ & \makecell{Source\\supp.}\\
\midrule
\multicolumn{6}{@{}l}{\emph{Gold F1 used in Stage 0/1}}\\
\multirow{4}{*}{Gold F1}
& None & 57.6 & 48.9 & 0.255 & 65.9 \\
& Gold F1 & \textbf{\fbox{59.5}} & \textbf{\fbox{50.6}} & \textbf{0.236} & 67.4 \\
& Verifier & 58.8 & 49.8 & 0.241 & 66.9 \\
& Judge & 59.1 & 50.1 & 0.239 & \textbf{67.5} \\
\cmidrule(l){1-6}
\multicolumn{6}{@{}l}{\emph{No gold F1 in any training stage}}\\
Verifier & Verifier & 58.1 & 49.2 & 0.248 & 66.6 \\
\rowcolor{lightblue!60}Judge & Judge & 58.6 & 49.6 & 0.245 & 67.1 \\
\bottomrule
\end{tabular}
\end{minipage}
\par\smallskip
\begin{minipage}{\linewidth}
\footnotesize
\textit{Grounding (Table~\ref{tab:grounding_main}).} Raw/Summary denote Always-Raw/Always-Summary; Full is Full \ours. F1 is answer F1; gold recall measures annotated evidence in the prompt. Supp./Unsup. denote support/unsupported rates from a fixed NLI-style evaluator (threshold 0.70). Source support/unsupported rate cover all examples; prompt support uses each policy's Summary/Raw subset and is descriptive, not paired. A balanced 200-example human audit gives 87.0\% agreement (Cohen's $\kappa=0.74$).
\par
\textit{Proxy rewards (Table~\ref{tab:proxy_reward_main}).} Verifier/Judge denote the frozen verifier/LLM judge; None is Stage~1 only. Canonical Cached evaluation: five-task macro F1/EM use held-out gold labels (\%); $\bar C$ counts selected-answer k tokens/query. An independent frozen evaluator scores source support (\%). Proxy runs reuse hyperparameters selected on gold-labeled development data, as in the gold-reward experiments.
\end{minipage}
\end{table}

Full improves HotpotQA F1 and lowers unsupported answers: gold evidence recall falls from 91.2 to 64.8 while source support rises from 61.9 to 65.3 (Table~\ref{tab:grounding_main}). Selective evidence delivery can therefore improve answer grounding even with less annotated evidence in the prompt. Fixed Summary leads on all three outcomes for FEVER. These answer-level metrics do not establish reasoning-trace faithfulness, and the generated outputs contain no citations.

A frozen judge supplies all training rewards and reaches 58.6 F1, within 0.9 F1 and 0.3 source-support points of gold-reward training, at 0.245k versus 0.236k cached tokens (Table~\ref{tab:proxy_reward_main}). Gold warm-start with judge-based Stage~2 narrows the F1 gap to 0.4. Proxy rewards thus retain most of the gains; hyperparameter selection still uses labeled development data.

Cost weights $(\lambda_{\rm in},\lambda_{\rm out})=(0.10,0.20)$ are tuned on Qwen HotpotQA development data and fixed thereafter. A Fresh Online sweep from $(0.05,0.00)$ to $(0.20,0.20)$ yields Qwen Full 59.9 F1/0.410k tokens to 58.7/0.358k, and Lite 59.2/0.268k to 58.0/0.216k. Mistral Full spans 50.5/0.430k to 49.3/0.373k, and Lite 49.6/0.287k to 48.4/0.229k; Full's first point exceeds Raw's 0.428k. A fallback requires a positive lower bound on the paired 95\% development utility-difference interval against the best fixed arm. It selects \ours for HotpotQA/2Wiki/MuSiQue and Summary for PopQA/FEVER: 59.7 F1/50.6 EM/0.235k cached tokens versus Full's 59.5/50.6/0.236k, using labeled development data to select a suitable task-specific operating point.

\subsection{Thinking Controls and Transfer}
\label{sec:joint_thinking}

\begin{table}[!htbp]
\caption{\textbf{Joint support and thinking control.} Task F1 (\%) and Cached cost (k/query). Six-arm \ours excludes Think-Low/High; AdaReasoner+~\citep{wang2025adareasoner} fixes support per task. Bold/boxes mark best task/macro F1 within each host.}
\label{tab:joint_thinking}
\centering
\small
\setlength{\tabcolsep}{5pt}
\renewcommand{\arraystretch}{1.0}
\begin{tabular}{@{}llrrrrrrr@{}}
\toprule
\multirow{2}{*}{Host} & \multirow{2}{*}{Policy} & \multicolumn{5}{c}{Task F1 (\%)} & \multirow{2}{*}{\makecell{Macro\\F1}} & \multirow{2}{*}{$\bar C$}\\
\cmidrule(lr){3-7}
& & HQA & 2Wiki & MSQ & PopQA & FEVER & & \\
\midrule
\multirow{6}{*}{\makecell[l]{Qwen3-8B-\\Thinking}}
& Raw+NoThink & 58.0 & 41.2 & 24.5 & 84.5 & 80.0 & 57.6 & 0.485\\
& Raw+Think-High & 62.5 & 47.0 & 30.8 & 87.0 & 89.5 & 63.4 & 1.280\\
& Direct+Think-High & 38.0 & 30.5 & 18.5 & 38.0 & 67.5 & 38.5 & 0.890\\
& AdaReasoner+ & 62.5 & 46.5 & 31.0 & \textbf{88.0} & 90.5 & 63.7 & 0.952\\
& \ours, 6-arm & 60.4 & 42.7 & 26.4 & 87.8 & 82.5 & 60.0 & 0.236\\
\rowcolor{lightblue!60}\cellcolor{white} & \ours, 12-arm & \textbf{64.8} & \textbf{49.2} & \textbf{33.5} & \textbf{88.0} & \textbf{91.0} & \textbf{\fbox{65.3}} & 0.690\\
\midrule
\multirow{6}{*}{\makecell[l]{Claude-4-\\Sonnet}}
& Raw+NoThink & 65.5 & 47.8 & 30.5 & 89.5 & 86.0 & 63.9 & 0.482\\
& Raw+Think-High & 70.5 & 54.0 & 37.0 & 91.5 & 95.0 & 69.6 & 1.350\\
& Direct+Think-High & 49.5 & 36.5 & 25.5 & 49.5 & 75.5 & 47.3 & 0.910\\
& AdaReasoner+ & 70.0 & 53.5 & 36.5 & 92.5 & 95.0 & 69.5 & 0.985\\
& \ours, 6-arm & 67.0 & 49.0 & 32.0 & 92.0 & 88.5 & 65.7 & 0.235\\
\rowcolor{lightblue!60}\cellcolor{white} & \ours, 12-arm & \textbf{72.4} & \textbf{55.5} & \textbf{38.8} & \textbf{93.0} & \textbf{95.5} & \textbf{\fbox{71.0}} & 0.715\\
\midrule
\multirow{6}{*}{DeepSeek-R1}
& Raw+NoThink & 64.0 & 46.5 & 31.5 & 92.0 & 86.5 & 64.1 & 0.480\\
& Raw+Think-High & 68.5 & 53.0 & 38.0 & 93.5 & 94.5 & 69.5 & 1.420\\
& Direct+Think-High & 50.5 & 38.0 & 27.5 & 50.0 & 76.0 & 48.4 & 0.940\\
& AdaReasoner+ & 68.5 & 53.0 & 38.0 & 93.5 & 95.0 & 69.6 & 1.045\\
& \ours, 6-arm & 65.5 & 47.5 & 32.0 & 93.5 & 88.0 & 65.3 & 0.234\\
\rowcolor{lightblue!60}\cellcolor{white} & \ours, 12-arm & \textbf{70.5} & \textbf{55.0} & \textbf{39.5} & \textbf{94.5} & \textbf{95.5} & \textbf{\fbox{71.0}} & 0.730\\
\bottomrule
\end{tabular}
\end{table}

With fixed retrieval and one pre-answer decision, the 12-arm router adds 1.4--1.9 macro F1 over Raw+Think-High at roughly half its cached cost (Table~\ref{tab:joint_thinking}). Gains concentrate on multi-hop tasks. Direct+Think-High reaches 38.5/47.3/48.4 F1 versus Raw+Think-High's 63.4/69.6/69.5: evidence still matters at the High budget, motivating joint control of both.

\begin{figure}[!htbp]
\centering
\includegraphics[width=\linewidth]{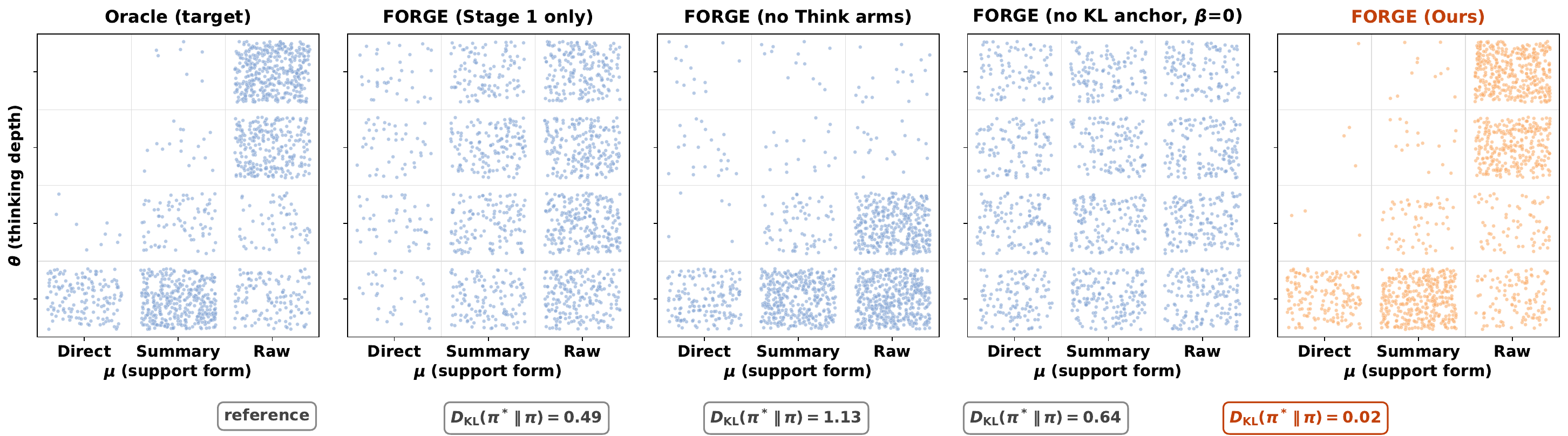}
\caption{\textbf{Joint action distributions} from 1,500 samples per policy. KL compares each distribution with the leftmost Boltzmann training target, distinct from the descriptive Oracle in Tables~\ref{tab:main} and~\ref{tab:nim_transfer_full}.}
\label{fig:policy_distribution}
\end{figure}
The full policy reaches $D_{\rm KL}=0.02$ from the Boltzmann target, versus 0.49 for Stage~1 and 0.64 without the KL anchor (Figure~\ref{fig:policy_distribution}). This supports the anchor's role in retaining the target's selective allocation across joint support--thinking combinations.

\subsection{Cross-Host Transfer}
\label{sec:transfer}

\begin{table}[!htbp]
\caption{\textbf{Zero-shot Cached transfer.} A Qwen-trained Full router with no target-host Stage~2. Cost excludes fresh probes. Bold marks best F1/cost per host; Appendix~\ref{app:nim_transfer_full} gives per-task results.}
\label{tab:nim_transfer}
\centering
\small
\renewcommand{\arraystretch}{1.0}
\setlength{\tabcolsep}{6pt}
\begin{tabular*}{0.88\linewidth}{@{\extracolsep{\fill}}lrrrr@{}}
\toprule
\multirow{2}{*}{Target host} & \multicolumn{2}{c}{Macro F1 (\%)} & \multicolumn{2}{c}{Cached $\bar C$ (k)}\\
\cmidrule(lr){2-3}\cmidrule(l){4-5}
& Raw & \ours & Raw & \ours\\
\midrule
Llama-3.3-70B & 58.5 & \textbf{61.2} & 0.318 & \textbf{0.205}\\
DeepSeek-V3.2 & \textbf{63.4} & \textbf{63.4} & 0.287 & \textbf{0.176}\\
Qwen3.5-397B & 63.1 & \textbf{65.6} & 0.305 & \textbf{0.198}\\
\bottomrule
\end{tabular*}
\end{table}

Full matches or exceeds Raw in 11/15 task--host cells, with macro F1 gains of +2.7/+0.0/+2.5 and 35.5\%/38.7\%/35.1\% fewer cached tokens (Table~\ref{tab:nim_transfer}). One source-trained router improves the quality--cost balance across host families and scales; DeepSeek ties Raw at the reported precision with the largest token saving. The online Lite result supports no-probe API use. Target-host calibration could further tailor Direct support to each host's prior knowledge.

\section{Conclusion}

\ours jointly controls evidence form and reasoning depth around a frozen host. Its factorized router distills an entropy-regularized quality--cost target and refines decisions through sampled host feedback. Across five tasks, Full improves nine-run cached macro F1 over Raw by 2.4/2.9 points on Qwen/Mistral; matched controls attribute 1.8/1.7 points to Stage~2. Lite improves local-host F1 and token cost with one online call, while Full uses four parallel probes for further accuracy. FEVER favors Summary; DeepSeek transfer primarily saves tokens. Results across host families establish joint evidence-form and thinking control as a practical adaptation mechanism with explicit quality, token, and latency tradeoffs across host families and deployment settings.

\setlength{\bibsep}{3pt plus 1pt minus 0.5pt}
\bibliography{references}
\bibliographystyle{iclr2027_conference}

\clearpage
\raggedbottom
\renewcommand{\thefigure}{S\arabic{figure}}
\renewcommand{\theHfigure}{S\arabic{figure}}
\setcounter{figure}{0}
\renewcommand{\thetable}{S\arabic{table}}
\renewcommand{\theHtable}{S\arabic{table}}
\setcounter{table}{0}

\renewcommand\appendixpagename{\centering\noindent\rule{\textwidth}{2pt} \LARGE Technical Appendices \\ \normalsize \noindent\rule{\textwidth}{1pt}}

\begin{appendices}

\appendix
\onecolumn
\appendixpage

\begingroup
\footnotesize
\setlength{\parskip}{0pt}
\setlength{\baselineskip}{10pt}
\makeatletter
\renewcommand*{\l@section}[2]{%
  \addpenalty{\@secpenalty}%
  \vskip 2pt
  \begingroup\bfseries\@dottedtocline{1}{0em}{1.6em}{#1}{#2}\endgroup}
\renewcommand*{\l@subsection}{\@dottedtocline{2}{1.6em}{2.7em}}
\@starttoc{toc}
\makeatother
\endgroup
\clearpage
\addtocontents{toc}{\protect\setcounter{tocdepth}{2}}

\section{Related Work}
\label{app:related}

Section~\ref{sec:intro} framed single-step, pre-answer routing around a frozen language-model host as a design space with one intrinsic axis (how deeply the model reasons, when the host exposes such a control) and three extrinsic axes (which model to call, how aggressively to retrieve, and what form the retrieved support takes).
The three subsections below organize prior work along this taxonomy.
Section~\ref{sec:rw_intrinsic} covers intrinsic adaptive thinking, Section~\ref{sec:rw_extrinsic} covers the two extrinsic axes that have received sustained attention, and Section~\ref{sec:rw_support} covers the support-form axis that FORGE completes.

\subsection{Intrinsic Adaptive Thinking}
\label{sec:rw_intrinsic}

A first line of work allocates inference compute inside the backbone itself.
Frontier reasoning systems learn to expend more decoding tokens on harder queries through chain-of-thought training~\citep{openai2024o1, deepseek2025r1}, and recent assistants expose this as a user-facing thinking budget that is set per request~\citep{anthropic2025claude37}.
A growing body of work studies how to control this budget automatically rather than by user toggle.
\citet{muennighoff2025s1} show that a simple budget-forcing trick at test time recovers most of the gains of full reasoning training on a small base model.
\citet{wang2025adareasoner} train an RL controller that selects reasoning configurations (temperature, depth) per query for any backbone, with theoretical convergence guarantees.
The recent survey of \citet{alomrani2025reasoningbudget} catalogs this space along an L1 (controllability) versus L2 (adaptiveness) axis and lists more than a dozen budget controllers built on top of reasoning-trained hosts.

All of these methods presuppose either a reasoning-trained backbone or a user-exposed thinking toggle, and they condition behavior on query difficulty by changing how the model itself executes.
Frozen hosts deployed behind closed APIs or pinned checkpoints cannot generally be retrained by the application, although some expose a per-request reasoning budget.
FORGE routes support form and exposed thinking options around a fixed host; our evaluation covers single-step decisions.

\subsection{Extrinsic Adaptivity: Model and Retrieval Routing}
\label{sec:rw_extrinsic}

The first extrinsic axis routes queries among backbones of different cost and capability.
FrugalGPT~\citep{chen2024frugalgpt} cascades a query through increasingly expensive models and stops when a verifier reports sufficient confidence.
RouteLLM~\citep{ong2024routellm} trains a preference-based router that selects a strong or weak model per query.
AutoMix~\citep{madaan2024automix} adds self-verification to estimate answer reliability before escalation.
BEST-Route~\citep{ding2025bestroute} jointly routes among models and response sample counts under an explicit cost-quality Pareto objective, the closest prior work in spirit to our framing.
These methods also study cost--quality routing, but their decision variable is the backbone itself. FORGE fixes the host and routes the support form for each query.

The second extrinsic axis decides whether and how aggressively to retrieve.
Retrieval-augmented generation prepends retrieved passages to every query~\citep{lewis2020rag, shi2024replug}; adaptive variants relax this default along several sub-axes.
Self-RAG~\citep{asai2024selfrag} trains the language model to emit reflection tokens that trigger retrieval on demand, at the cost of modifying the host.
CRAG~\citep{yan2024crag} adds a post-hoc evaluator that triggers corrective retrieval when document quality is low, and IRCoT~\citep{trivedi2023interleaving} interleaves retrieval steps with chain-of-thought reasoning for multi-hop queries.
Adaptive-RAG~\citep{jeong2024adaptiverag} trains a query-complexity classifier to pick among no retrieval, single-step retrieval, and iterative multi-step retrieval.
Self-Route~\citep{li2024selfroute} frames the decision as a per-query binary choice between RAG and long-context generation.
More recent routing benchmarks~\citep{wang2026ragrouterbench, bansal2026lightweight, zhou2026selectthensolve} catalog inference-time strategies such as Direct, CoT, and ReAct, and study how to route among them.
The \emph{sufficient context} lens of \citet{joren2025sufficient} treats sufficiency as a property of query--context pairs and uses it to gate abstention. Its analysis of frozen hosts motivates treating support form as a separate decision about evidence delivery.

Both axes operate on \emph{whether} or \emph{how deeply} to retrieve.
Retrieved content typically reaches the host as raw passages or a near-equivalent, leaving support form fixed.
Adaptive-RAG~\citep{jeong2024adaptiverag} is evaluated in Section~\ref{sec:exp}, and a host-confidence Self-Routing heuristic is reported in Appendix~\ref{app:heuristic_baselines}.

\subsection{Support-Form Routing with a Fixed Host}
\label{sec:rw_support}

The third extrinsic axis selects what \emph{form} the external support takes, holding the backbone and the retrieval backend fixed. Prior tiered-memory systems already route between two support forms within memory hierarchies.
TierMem~\citep{zhu2026tiermem} organizes memory into a two-tier hierarchy of summaries and raw pages with a sufficiency router that escalates from summary to raw, reducing tokens by 54\% on LoCoMo~\citep{maharana2024locomo} with minimal accuracy loss.
MemPO~\citep{li2026mempo} and Memento~\citep{zhou2025memento} apply reinforcement learning to optimize memory management, and PRIME~\citep{wang2026prime} builds experience libraries through iterative evolution without backbone training.
A different angle is offered by advisor models~\citep{asawa2025advisor}, which train a small policy that emits free-form natural-language steering instructions to a black-box host on a per-instance basis.
TierMem routes between summary and raw but does not include the Direct option that bypasses retrieved support; advisor models deliver free-form hints rather than choosing among the fixed Direct, Summary, and Raw support forms studied here.

\citet{jiang2025agenticadaptation} catalog the space of agent-supervised tool adaptation, in which a frozen agent supervises the training of a small downstream module via reward feedback, and place tiered-memory routers, advisor models, and similar systems in this family.
Two additional architectural precedents matter for FORGE's design, even though each was originally motivated by a different task.
\citet{jiang2025s3} introduce \emph{s3}, a PPO-trained T2 search agent that decouples the retriever from the generator and optimizes a Gain-Beyond-RAG reward; s3's choice of on-policy PPO over a frozen host establishes that small reward-based routers can train on frontier-scale hosts, and we include it as a learned-routing baseline in Section~\ref{sec:exp}.
\citet{sysformer2025} introduce \emph{Sysformer}, a transformer-based adapter that updates system prompts for frozen LLMs in the \emph{safety} setting (refusal of harmful prompts), and establish that attention-based adapters can steer frozen hosts without weight access; we retarget their adapter template to our support-form alphabet as a transformer-based learned-routing baseline, acknowledging that this retargeting is a partial repurposing rather than a like-for-like comparison.
FORGE studies a three-way Direct/Summary/Raw support choice around a frozen host, factorizes it with an optional thinking-depth choice (Section~\ref{sec:arms}), and refines the routing policy against measured host feedback. Its distinction from TierMem is the additional Direct arm and joint decision studied under the stated single-step task protocol, not the invention of support-form routing itself.
TierMem-style two-arm routing, s3's PPO-trained searcher, and the retargeted Sysformer adapter are all included as experimental baselines in Section~\ref{sec:exp}.

The three forms differ in prompt length and evidence content. FORGE compares their measured answer quality and token use under the specified utility; no information-bottleneck or rate-distortion optimality claim is needed for this empirical choice. A complementary line in agent context compression~\citep{kang2025acon} condenses observation histories for long-horizon agents. FORGE studies the narrower, single-answer decision and does not evaluate long-horizon interactions.

\section{Implementation Details}
\label{appendix:implementation}

All experiments were conducted on an HPC cluster (details anonymized for double-blind review).
Local open-source hosts were served on 64 GB HBM GPU accelerators, with up to \textbf{128 accelerators running concurrently} across 16 nodes at peak.
Closed-source and API-served hosts were queried through the corresponding managed inference endpoints.
The final experimental matrix spans \textbf{8 LLM backbones} ranging from 7B to 671B parameters, \textbf{5 benchmarks}, and \textbf{40 host--benchmark cells}.
Stage 1 supervised training is CPU-only and completes in under two minutes; Stage 2 GRPO refinement runs on a single 8-GPU node and completes in 8 to 12 hours per Pareto operating point.

\subsection{Hardware and Software Environment}
\label{appendix:env}

Table~\ref{tab:env} lists the hardware and software used in our experiments.

\begin{table}[!htb]
\small
\renewcommand{\arraystretch}{1.06}
\centering
\caption{Hardware and software environment.}
\label{tab:env}
\small
\begin{tabular}{ll}
\toprule
\textbf{Category} & \textbf{Value} \\
\midrule
\multicolumn{2}{l}{\textit{Cluster}} \\
\quad System             & HPC cluster (anonymized for review) \\
\quad Node CPU           & 64-core x86\_64 \\
\quad Node GPUs          & 8 $\times$ 64 GB HBM accelerators \\
\quad Node memory        & 512 GB \\
\midrule
\multicolumn{2}{l}{\textit{Software stack}} \\
\quad OS                 & Linux \\
\quad GPU runtime        & 6.2.4 (vendor-specific) \\
\quad Python             & 3.11.11 \\
\quad PyTorch            & 2.6.0 (with vendor-specific GPU backend) \\
\quad Transformers       & 5.0.0 \\
\quad Sentence-Transformers & 5.4.1 \\
\quad Datasets (HF)      & 4.0.0 \\
\quad scikit-learn       & 1.6.1 \\
\quad rank\_bm25         & 0.2.2 \\
\bottomrule
\end{tabular}
\end{table}

\subsection{Frozen Host Configuration}
\label{appendix:hosts}

The host pool is partitioned into three groups by deployment mode and intrinsic capability (Table~\ref{tab:hosts}).
Locally hosted models are loaded in float16 and sharded across eight HBM accelerators per node via \texttt{device\_map="auto"}.
Qwen3-8B-Instruct and Qwen3-8B-Thinking denote the same Qwen3-8B checkpoint run with its thinking mode disabled and enabled, respectively.
API models are accessed through managed inference endpoints without weight access. The cross-host API tables use cached-feature evaluation; only Lite has a measured Fresh Online API example, and Full FORGE's Fresh Online API latency is unmeasured.

\begin{table}[!htb]
\small
\renewcommand{\arraystretch}{1.06}
\centering
\caption{Frozen LLM backbones grouped by deployment mode and thinking capability.}
\label{tab:hosts}
\small
\begin{tabular}{llcll}
\toprule
\textbf{Host} & \textbf{Family} & \textbf{Parameters} & \textbf{Precision} & \textbf{Deployment} \\
\midrule
\multicolumn{5}{l}{\textit{Non-thinking local hosts (main results, Table~\ref{tab:main})}} \\
Qwen3-8B-Instruct (main)     & Qwen 3      & 8B       & float16 & Local (8 accelerators) \\
Mistral-7B-Instruct-v0.3     & Mistral     & 7B       & float16 & Local (8 accelerators) \\
\midrule
\multicolumn{5}{l}{\textit{Thinking-capable hosts (composition experiment, Table~\ref{tab:joint_thinking})}} \\
Qwen3-8B-Thinking            & Qwen 3      & 8B       & float16 & Local (8 accelerators) \\
Claude Sonnet 4              & Anthropic   & --       & --      & API \\
DeepSeek-R1                  & DeepSeek    & 671B MoE & --      & API \\
\midrule
\multicolumn{5}{l}{\textit{Frontier non-thinking hosts (cross-host transfer, Table~\ref{tab:nim_transfer})}} \\
Llama-3.3-70B-Instruct       & Llama 3.3   & 70B      & --      & API \\
DeepSeek-V3.2                & DeepSeek    & 671B MoE & --      & API \\
Qwen3.5-397B-A17B            & Qwen 3.5    & 397B MoE & --      & API \\
\bottomrule
\end{tabular}
\end{table}

Table~\ref{tab:decoding} lists the decoding parameters.
All non-thinking arms share identical decoding to isolate the effect of the action from decoding variability.
Self-consistency (SC) probes use temperature sampling with a short per-call budget; the full online cost still includes all three samples and the greedy probe.
Think-Low and Think-High set the host's per-request reasoning budget (construction in Appendix~\ref{appendix:retrieval}).

\begin{table}[!htb]
\small
\renewcommand{\arraystretch}{1.06}
\centering
\caption{Decoding and inference parameters across host groups.}
\label{tab:decoding}
\small
\begin{tabular}{lll}
\toprule
\textbf{Parameter} & \textbf{Local hosts} & \textbf{API hosts} \\
\midrule
\texttt{max\_new\_tokens} (arm answer) & 64  & 64 \\
\texttt{max\_new\_tokens} (SC sample)   & 24  & 24 \\
Greedy temperature                      & 0   & 0 \\
SC temperature                          & 0.7 & 0.7 \\
SC top-$p$                              & 0.9 & 0.9 \\
SC samples $K_\mathrm{sc}$              & 3   & 3 \\
Think-Low budget (output tokens)        & 1{,}024 & 1{,}024 \\
Think-High budget (output tokens)       & 4{,}096 & 4{,}096 \\
API rate limit used                     & --  & 25 requests/min \\
\bottomrule
\end{tabular}
\end{table}

\subsection{Retrieval and Action Construction}
\label{appendix:retrieval}

All actions share a fixed retrieval backend so that performance differences reflect routing rather than retriever tuning (Table~\ref{tab:retrieval}).
Each action $a=(\mu,\theta)$ selects support form $\mu$ and thinking setting $\theta$.

\paragraph{Summary form construction.}
The Summary action is constructed deterministically without any auxiliary language model: BM25 retrieves the top-$3$ passages, all sentences within those passages are re-ranked by BM25 against the query, and the top-ranked sentences are concatenated until a budget of $\sim$$140$ words is reached.
No second-stage LLM call or learned compressor is invoked for Summary, so its selected-answer input tokens are its host prompt length. This statement concerns support construction; Full FORGE's pre-routing host probes are accounted for separately under the fresh-online protocol in Appendix~\ref{app:latency}.
We use fixed BM25 and deterministic extractive Summary. Learned compressors and other retrieval backends require separate training and evaluation.

\paragraph{Thinking-setting construction.}
NoThink uses plain decoding without a thinking instruction, and CoT-Prompt appends the suffix in Table~\ref{tab:retrieval}; both settings are available on every host. On thinking-capable hosts, Think-Low and Think-High set the host's per-request reasoning budget to $1{,}024$ and $4{,}096$ output tokens, respectively; all other decoding parameters follow Table~\ref{tab:decoding}.

\begin{table}[!htb]
\small
\renewcommand{\arraystretch}{1.06}
\centering
\caption{Retrieval pipeline and action construction.}
\label{tab:retrieval}
\small
\begin{tabular}{ll}
\toprule
\textbf{Component} & \textbf{Value} \\
\midrule
Retriever                     & BM25 (\texttt{rank\_bm25}) \\
Embedding model               & BAAI/bge-base-en-v1.5 (768-d, normalized) \\
Tokenizer                     & lowercase + punctuation removal \\
Stopword list                 & 57 common English stopwords \\
Top-$k$ for Summary/Raw       & 3 passages \\
\midrule
\multicolumn{2}{l}{\textit{Support-form axis $\mathcal{M}$}} \\
\quad Direct ($\mu = 0$)      & Question only, 0 extra input tokens \\
\quad Summary ($\mu = 1$)     & Query-aware sentence selection; ${\le}140$ words \\
\quad Raw ($\mu = 2$)         & Top-3 BM25 passages; ${\le}200$ words per passage \\
\midrule
\multicolumn{2}{l}{\textit{Thinking-depth axis $\Theta$ (thinking-capable hosts)}} \\
\quad NoThink                 & Plain decoding; no thinking instruction \\
\quad CoT-Prompt              & ``Let us think step by step'' suffix \\
\quad Think-Low               & Reasoning budget $1{,}024$ output tokens \\
\quad Think-High              & Reasoning budget $4{,}096$ output tokens \\
\bottomrule
\end{tabular}
\end{table}

\subsection{Router Architecture and Three-Stage Training}
\label{appendix:router}

Full FORGE uses a factorized MLP with a $789 \to 256 \to 256$ shared encoder, followed by support-form and thinking-depth heads (Section~\ref{sec:router}). FORGE-Lite retrains the router on the 778 host-independent input dimensions.
Training proceeds in three stages: offline arm enumeration on the warm-start subset (Stage 0), supervised KL distillation against the Boltzmann target (Stage 1), and policy-guided group-relative refinement using host feedback on the expanded action alphabet (Stage 2).
Table~\ref{tab:router} lists all router-specific hyperparameters.

\begin{figure}[!htbp]
\centering
\begin{tikzpicture}[>=stealth, font=\footnotesize,
  box/.style={draw=natureblue!65, rounded corners=2pt, fill=natureblue!5,
              minimum height=1.85cm, align=center, inner sep=4pt},
  arrow/.style={->, thick, draw=natureblue!80}]
\node[box, text width=3.65cm] (features) at (1.95,0)
  {\textbf{Pre-answer features}\\[3pt]
   BGE query vector (768-d)\\
   Full: 21 structured features\\
   Lite: 10 structured features};
\node[box, text width=3.1cm] (encoder) at (6.25,0)
  {\textbf{Shared encoder}\\[2pt]
   Full: $789\!\to\!256\!\to\!256$\\
   Lite: $778\!\to\!256\!\to\!256$};
\node[box, text width=4.15cm] (heads) at (10.9,0)
  {\textbf{Factorized policy}\\[3pt]
   $\pi^{\mu}$: three support forms\\
   $\pi^{\theta}$: reasoning given support\\
   Full: $\approx269$K parameters};
\node[box, text width=4.15cm, minimum height=0.75cm] (host) at (10.9,-2.0)
  {Selected prompt $\longrightarrow$ frozen host};
\draw[arrow] (features.east) -- (encoder.west);
\draw[arrow] (encoder.east) -- (heads.west);
\draw[arrow] (heads.south) -- (host.north);
\end{tikzpicture}
\caption{\textbf{Full and Lite router deployment paths.} Full uses four pre-routing host generations to construct its 11 host-derived features on a fresh query; Lite is retrained without them. Both variants select one joint action before the frozen host answers, and the host weights remain fixed.}
\label{fig:router_deployment}
\end{figure}
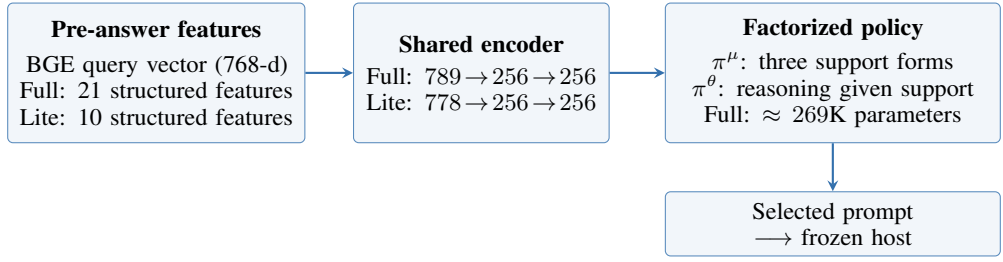

\paragraph{Stage-1 warm-start objective (full form).}
The warm-start router $\pi_{\theta_0}$ is trained by minimizing the KL divergence on the $\mu$ head against the Boltzmann target $\pi^*_\mathrm{Boltz}$ restricted to $\mathcal{A}_\mathrm{warm}$ at temperature $\tau_0{=}1.0$:
\begin{equation}
\label{eq:kl}
\mathcal{L}_\mathrm{warm}(\theta) \;=\; \frac{1}{N_\mathrm{warm}} \sum_{i} D_\mathrm{KL}\!\big(\pi^*_\mathrm{Boltz}(\cdot \mid x_i) \,\big\|\, \pi_\theta^\mu(\cdot \mid h(x_i))\big),
\end{equation}
with $\pi_\theta^\theta$ initialized uniform.
The target is the closed-form optimum of the stated entropy-regularized finite-action utility on $\mathcal{A}_\mathrm{warm}$; the finite-capacity Stage-1 router only approximates this target.
The default warm-start uses three actions. The controlled Stage-2 comparison also evaluates an additional full-alphabet, six-action Stage-1 control under matched features, decoding, and cost weights; that control is distinct from the three-action warm-start row in the canonical table.

\paragraph{Stage-2 clipped policy and auxiliary losses (full form).}
Let $\rho_{b,g}(\theta) = \pi_\theta(a_{b,g} \mid x_b) / \pi_{\theta_\mathrm{old}}(a_{b,g} \mid x_b)$ be the displayed old-policy ratio. The clipped surrogate uses the PPO-style construction of~\citet{schulman2017ppo} with the group-relative advantage $\hat A_{b,g}$ of~\citet{shao2024deepseekmath} (Eq.~\ref{eq:advantage}):
\begin{equation}
\label{eq:grpo_policy}
\mathcal{L}_\mathrm{pol}(\theta) \;=\; -\frac{1}{BG}\sum_{b, g} \min\!\Big(\rho_{b,g}\,\hat A_{b,g},\;\; \mathrm{clip}(\rho_{b,g}, 1-\eta, 1+\eta)\,\hat A_{b,g}\Big),
\end{equation}
with $\eta = 0.2$.
The action groups are stratified to cover support forms when $G \geq |\mathcal{M}|$ (Table~\ref{tab:router}). This changes the behavior proposal to a distribution $q(a\mid x)$ that need not equal $\pi_{\theta_\mathrm{old}}(a\mid x)$. Without a proposal correction involving $q$, the displayed ratio does not make Eq.~\ref{eq:grpo_policy} an unbiased policy-gradient estimator for the old policy. We therefore treat Stage 2 as a policy-guided, group-relative clipped surrogate aligned with measured utility, not an exact gradient step toward or convergence guarantee for the Boltzmann target.
The total Stage-2 loss is $\mathcal{L}_\mathrm{GRPO} = \mathcal{L}_\mathrm{pol} + \mathcal{L}_\mathrm{KL} + \mathcal{L}_\mathrm{ent}$ with the auxiliary terms
\begin{align}
\label{eq:grpo_kl}
\mathcal{L}_\mathrm{KL}(\theta) \;&=\; \beta \cdot \frac{1}{B} \sum_b D_\mathrm{KL}\!\big(\pi_\theta(\cdot \mid x_b) \,\big\|\, \pi_{\theta_0}(\cdot \mid x_b)\big), \\
\label{eq:grpo_ent}
\mathcal{L}_\mathrm{ent}(\theta) \;&=\; -\alpha \cdot \frac{1}{B} \sum_b H\!\big(\pi_\theta(\cdot \mid x_b)\big).
\end{align}
The KL anchor keeps the clipped update close to the Stage-1 warm-start reference $\pi_{\theta_0}$; the entropy bonus discourages premature collapse on the larger alphabet.
Optimization uses AdamW at learning rate $1 \times 10^{-5}$ with $K_\mathrm{inner} = 4$ minibatch updates per rollout before $\pi_{\theta_\mathrm{old}}$ is refreshed.
The KL coefficient $\beta$ follows an adaptive schedule: it is multiplied by $1.5$ every $100$ steps when the observed per-query KL exceeds $0.05$, and divided by the same factor when it falls below $0.005$, so that the policy neither drifts off the warm-start manifold nor stagnates on it.
Stage 0 costs $|\mathcal{A}_\mathrm{warm}| \cdot N_\mathrm{warm} \approx 3 \cdot 1500 = 4500$ host calls (single node, under one hour); Stage 1 is CPU-only and completes in under two minutes. Each Stage-2 operating point uses $T_\mathrm{grpo} B G = 5000 \times 32 \times 8 = 1{,}280{,}000$ atomic model completions, about 384M input and 82M output tokens in the measured resource accounting, and 8--12 hours on one 8-GPU node (64--96 GPU-hours).
BGE and MLP routing take $0.5$--$5.5$ ms in the measured settings, excluding host calls that acquire Full FORGE's probes on unseen queries. Appendix~\ref{app:latency} includes those calls in Fresh Online token and latency accounting.
\begin{table}[!htb]
\small
\renewcommand{\arraystretch}{1.06}
\centering
\caption{Router architecture and three-stage training hyperparameters.}
\label{tab:router}
\small
\begin{tabular}{@{}p{0.42\columnwidth}p{0.54\columnwidth}@{}}
\toprule
\textbf{Hyperparameter} & \textbf{Value} \\
\midrule
\multicolumn{2}{l}{\textit{Architecture}} \\
\quad Shared encoder $h(x)$         & 2-layer MLP, $789 \to 256 \to 256$, ReLU, dropout 0.1 \\
\quad Support-form head $\pi^\mu$   & Linear, 3 logits \\
\quad Thinking-depth head $\pi^\theta$ & Linear, $|\Theta|$ logits, input $[h(x); e(\mu)]$ with $e(\mu) \in \mathbb{R}^{16}$ \\
\quad Total parameter count         & ${\sim}$269K \\
\midrule
\multicolumn{2}{l}{\textit{Stage 0: Offline arm enumeration}} \\
\quad Warm-start alphabet $\mathcal{A}_\mathrm{warm}$ & $\mathcal{M} \times \{\mathrm{NoThink}\}$ (3 arms) \\
\quad Warm-start size $N_\mathrm{warm}$               & 1{,}500 queries \\
\midrule
\multicolumn{2}{l}{\textit{Stage 1: Supervised warm-start (KL distillation)}} \\
\quad Soft-label temperature $\tau_0$   & 1.0 \\
\quad Optimizer                          & AdamW \\
\quad Learning rate                      & $2 \times 10^{-4}$ \\
\quad Weight decay                       & 0.01 \\
\quad Batch size                         & 64 \\
\quad Maximum epochs                     & 50 \\
\quad Early-stopping patience            & 7 (dev accuracy) \\
\midrule
\multicolumn{2}{l}{\textit{Stage 2: policy-guided GRPO refinement}} \\
\quad Full alphabet $\mathcal{A}$        & $K{=}6$ (non-thinking hosts) or $K{=}12$ (thinking hosts) \\
\quad Rollout group size $G$             & 8 \\
\quad Batch size $B$ (queries per step)  & 32 \\
\quad Inner gradient steps per rollout   & $K_\mathrm{inner} = 4$ \\
\quad PPO clip range $\eta$              & 0.2 \\
\quad KL anchor target                   & 0.02 (per query, adaptive $\beta$) \\
\quad KL anchor $\beta$ initial          & 0.05; multiplied by 1.5 every 100 steps if KL out of band \\
\quad Entropy bonus $\alpha$             & 0.01 \\
\quad Sampling                           & Stratified to ensure $\mu$ coverage when $G \geq |\mathcal{M}|$ \\
\quad Optimizer                          & AdamW \\
\quad Learning rate                      & $1 \times 10^{-5}$ \\
\quad Total update steps $T_\mathrm{grpo}$ & 5{,}000 \\
\midrule
\multicolumn{2}{l}{\textit{Pareto utility (default)}} \\
\quad Input cost weight $\lambda_\mathrm{in}$  & 0.1 (swept in Table~\ref{tab:lambda}) \\
\quad Output cost weight $\lambda_\mathrm{out}$ & 0.2 (swept in Table~\ref{tab:lambda}) \\
\bottomrule
\end{tabular}
\end{table}

\subsection{Feature Extraction}
\label{appendix:features}

Full FORGE's input $\phi_{\mathrm{Full}}(x) \in \mathbb{R}^{789}$ combines six feature groups (Table~\ref{tab:features}). Eleven dimensions depend on a greedy host probe and three self-consistency samples, so an unseen fresh-online query requires four pre-routing host calls before its routed answer.
FORGE-Lite is retrained with the 778 host-independent dimensions only: 768 BGE embedding dimensions, four BM25 statistics, one query--top-passage cosine, and five query-structure features. It uses no host probe in training or inference. Structured features are standardized using training-set statistics.

\begin{table}[!htb]
\small
\renewcommand{\arraystretch}{1.06}
\centering
\caption{Feature groups for Full FORGE (789 dimensions) and FORGE-Lite (778 dimensions). Only Full uses the 11 host-dependent dimensions.}
\label{tab:features}
\small
\setlength{\tabcolsep}{4pt}
\renewcommand{\arraystretch}{1.08}
\begin{tabular}{@{}>{\raggedright\arraybackslash}p{0.24\linewidth}>{\raggedright\arraybackslash}p{0.62\linewidth}r@{}}
\toprule
\textbf{Group} & \textbf{Features} & \textbf{Dim} \\
\midrule
Query embedding (BGE)  & Frozen sentence embedding                                              & 768 \\
Query structural       & length, WH flag, entity count, comparison flag, temporal flag         & 5 \\
Retrieval              & BM25 top-1 score, top-5 mean, score gap, score std                    & 4 \\
Host probe             & answer length, output tokens, IDK flag, hedging flag, query overlap, numeric flag, confidence & 7 \\
Self-consistency       & agreement, unique ratio, average length, greedy match                  & 4 \\
Cross-modal            & query--top-1-passage cosine similarity                                 & 1 \\
\midrule
\textbf{Full total}    & All six groups                                                         & \textbf{789} \\
\textbf{Lite total}    & BGE + query structural + retrieval + cross-modal                     & \textbf{778} \\
\bottomrule
\end{tabular}
\end{table}

The strong BGE+BM25 baseline uses 773 host-independent dimensions (768 BGE, four BM25, and one query--top-passage cosine) and matches the Stage-0 data, MLP capacity, optimizer, and development-set search budget. It has no query-structure features, host probes, or GRPO. Table~\ref{tab:feature_ladder_app} reports the canonical cached feature-and-training ladder. Its point estimates should not be combined with the nine-run means in the uncertainty analysis as though they were the same statistic.

\begin{table}[!htb]
\small
\renewcommand{\arraystretch}{1.06}
\caption{\textbf{Canonical cached F1 (\%) for matched feature and training variants.} The 773-dimensional BGE+BM25 baseline supplies a competitive reference. FORGE-Lite adds the five structural features without host calls; Full FORGE adds 11 host-dependent dimensions. All entries are point estimates on the canonical evaluation split, separate from the nine-run means.}
\label{tab:feature_ladder_app}
\centering
\small
\begin{tabular}{lccc}
\toprule
Variant & Dimensions & Qwen3-8B & Mistral-7B \\
\midrule
BGE+BM25-Hard & 773 & 56.8 & 46.7 \\
BGE+BM25-KL & 773 & 57.5 & 47.5 \\
FORGE-773 (three arms) & 773 & 58.0 & 48.0 \\
FORGE-773 (six arms) & 773 & 58.4 & 48.6 \\
FORGE-Lite & 778 & 58.7 & 49.1 \\
Full FORGE & 789 & 59.5 & 50.1 \\
\bottomrule
\end{tabular}
\end{table}

\subsection{Dataset Configuration}
\label{appendix:datasets}

Table~\ref{tab:datasets} lists the train/test splits per host group.
The original tables use a canonical split (data-sampling seed 42), and cross-host transfer uses a matched canonical test subset. The uncertainty analysis combines three independent train/dev/test splits with three training seeds per split on both main local hosts. Its nine-run mean and standard deviation are distinct from canonical point estimates and paired-query bootstrap intervals on the canonical test split.

\begin{table}[!htb]
\small
\renewcommand{\arraystretch}{1.06}
\centering
\caption{Per-benchmark sample sizes for the canonical split (data-sampling seed 42). The additional nine-run uncertainty analysis uses three independent splits and three training seeds per split.}
\label{tab:datasets}
\small
\begin{tabular}{lccc}
\toprule
\textbf{Benchmark} & \textbf{Train (local)} & \textbf{Test (local)} & \textbf{Test (transfer)} \\
\midrule
HotpotQA (distractor)      & 1{,}500 & 300 & 75 \\
2WikiMultiHopQA            & 1{,}500 & 300 & 75 \\
MuSiQue                    & 1{,}500 & 300 & 75 \\
PopQA                      & 1{,}500 & 300 & 75 \\
FEVER                      & 1{,}500 & 300 & 75 \\
\bottomrule
\end{tabular}
\end{table}

\subsection{Asset Licenses and Attribution}
\label{appendix:licenses}

All datasets, frozen language models, and software libraries used in this work are credited to their original creators, accessed under their original licenses, and used in a manner consistent with their intended research use. Table~\ref{tab:licenses} summarizes the assets and their licenses; the underlying datasets and model weights are not redistributed with this work.

\begin{table}[!htb]
\small
\renewcommand{\arraystretch}{1.06}
\centering
\caption{Licenses for existing assets used in this work. Datasets are accessed via their original release channels; frozen LLMs are accessed via their published weights or provider APIs.}
\label{tab:licenses}
\footnotesize
\setlength{\tabcolsep}{4.5pt}
\begin{tabular}{llll}
\toprule
\textbf{Asset} & \textbf{Type} & \textbf{License / Terms} & \textbf{Citation} \\
\midrule
HotpotQA (distractor)      & Dataset & CC BY-SA 4.0           & \citet{yang2018hotpotqa} \\
2WikiMultiHopQA            & Dataset & Apache 2.0             & \citet{ho2020constructing} \\
MuSiQue                    & Dataset & CC BY 4.0              & \citet{trivedi2022musique} \\
PopQA                      & Dataset & MIT                    & --- \\
FEVER                      & Dataset & CC BY-SA 3.0           & --- \\
Wikipedia (retrieval corpus) & Corpus & CC BY-SA 4.0          & --- \\
\midrule
Qwen3-8B-Instruct / -Thinking & Frozen LLM & Apache 2.0        & --- \\
Qwen3.5-397B               & Frozen LLM & Apache 2.0           & --- \\
Mistral-7B-Instruct-v0.3   & Frozen LLM & Apache 2.0           & --- \\
Llama-3.3-70B-Instruct     & Frozen LLM & Llama 3.3 Community  & --- \\
DeepSeek-V3.2 / -R1        & Frozen LLM & MIT                  & --- \\
Claude Sonnet 4            & Frozen LLM (API) & Anthropic API terms & --- \\
\midrule
BGE-base-en-v1.5           & Embedder & MIT                    & \citet{xiao2023bge} \\
\texttt{rank\_bm25}        & Library & Apache 2.0              & --- \\
PyTorch / HuggingFace Transformers & Library & BSD / Apache 2.0 & --- \\
\bottomrule
\end{tabular}
\end{table}

\subsection{Compute Scale}
\label{appendix:compute}

Table~\ref{tab:compute} summarizes the overall compute footprint.
At peak, \textbf{128 GPU accelerators ran concurrently} across 16 nodes for Stage 0 arm enumeration on the local backbones.
Stage 2 GRPO refinement uses a single 8-GPU node per Pareto operating point.
API-based experiments with frontier proprietary models were executed on externally managed infrastructure and incur no local cluster compute time.
Stage 2 is a one-time source-host training cost, not a target-host cost incurred on each transfer query. In the matched six-action, nine-run comparison, it adds 1.8 macro-F1 points on Qwen3-8B and 1.7 on Mistral-7B over the corresponding Stage-1 checkpoints, while reducing cached selected-answer tokens by 7.1\% and 8.1\%, respectively. The larger three-arm-to-six-arm difference also includes action-alphabet expansion and must not be attributed wholly to GRPO. Skipping Stage 2 avoids rollout cost but leaves Full FORGE's Fresh Online host probes.

\begin{table}[!htb]
\small
\renewcommand{\arraystretch}{1.06}
\centering
\caption{Compute resources and experiment counts.}
\label{tab:compute}
\small
\begin{tabular}{lr}
\toprule
\textbf{Resource or artefact} & \textbf{Value} \\
\midrule
Cluster nodes used (unique)                & 16+ \\
\textbf{Peak concurrent GPU accelerators}  & \textbf{128} \\
SLURM allocations active                   & 2 institutional allocations \\
Stage 0 host calls (per local backbone)    & ${\sim}$4{,}500 (3 arms $\times$ 1{,}500 queries) \\
Stage 2 completions (per Pareto point)      & 1.28M ($T \cdot B \cdot G$ at $T{=}5000$, $B{=}32$, $G{=}8$) \\
Stage 2 input / output tokens (per point)   & About 384M / 82M \\
Stage 2 wall-clock (Qwen3-8B local, per Pareto point)      & 8 to 12 hours on 8-GPU node \\
Stage 2 local GPU-hours (per point)         & 64 to 96 GPU-hours \\
API calls issued (transfer + thinking)     & ${\sim}$22{,}000 \\
Distinct LLM backbones evaluated           & 8 \\
\quad Parameter range                      & 7B--671B \\
Benchmarks evaluated                       & 5 \\
Host $\times$ benchmark cells              & 40 \\
\bottomrule
\end{tabular}
\end{table}

The matched six-action budget trajectory separates GRPO refinement from action expansion. At 0, 2{,}500, and 5{,}000 updates, Qwen3-8B reaches 57.6, 58.9, and 59.4 macro-F1 with 0.255, 0.244, and 0.237k cached selected-answer tokens per query. Mistral-7B reaches 48.3, 49.5, and 50.0 macro-F1 with 0.272, 0.259, and 0.250k tokens. These improvements require the one-time training expenditure above; they do not account for Full FORGE's Fresh Online probes on later queries.

\subsection{Inference Latency}
\label{app:latency}

Table~\ref{tab:latency} isolates local feature and router computation. The frozen BGE encoder dominates this local component, with little MLP time, but the table does not include the host executions required to populate Full FORGE's 11 host-dependent features on a fresh query. Main-text Table~\ref{tab:online} reports end-to-end Fresh Online latency with those calls included.

\begin{table}[!htb]
\small
\renewcommand{\arraystretch}{1.06}
\centering
\caption{\textbf{Local routing latency (ms/query).} Median over 1{,}000 HotpotQA-style queries.}
\label{tab:latency}
\begin{threeparttable}
\small
\setlength{\tabcolsep}{4pt}
\begin{tabular}{@{}>{\raggedright\arraybackslash}p{0.46\linewidth}ccc@{}}
\toprule
\textbf{Component} & \makecell{\textbf{CPU}\\(1 thread)} & \makecell{\textbf{GPU}\\(batch 1)} & \makecell{\textbf{GPU}\\(batch 32, amort.)} \\
\midrule
BGE-base encoder (frozen)              & 27.4 & 4.9  & 0.39 \\
Structured feature extraction          &  0.6 & 0.6  & 0.05 \\
FORGE MLP head ($\pi^\mu$, $\pi^\theta$) &  1.1 & 0.05 & 0.01 \\
\midrule
\textbf{Total routing decision}        & \textbf{29.1} & \textbf{5.5} & \textbf{0.45} \\
\midrule
Self-consistency probe ($K_\mathrm{sc}{=}3$ host samples) & --- & --- & 90.0 \\
Host answer inference (Qwen3-8B, 64 out tokens) & --- & 122.5 & 78.4 \\
\bottomrule
\end{tabular}
\begin{tablenotes}[flushleft]
\footnotesize
\item Mean query length is 18 tokens. The total routing-decision row covers local computation only. On a fresh Full FORGE query, three self-consistency samples and one greedy probe require host calls; the self-consistency row is an amortized batch-32 component, not complete Fresh Online overhead. Host answer inference is shown for scale.
\end{tablenotes}
\end{threeparttable}
\end{table}

\noindent In the cached-feature setting, a previously processed query can reuse the four Full FORGE probe outputs, leaving only the local routing computation before the selected answer. A new query in Fresh Online use must acquire those features: one greedy probe and three self-consistency samples precede the routed answer, for five host calls in total. FORGE-Lite retrains on host-independent features and needs only one host call for the routed answer.

Full FORGE trades much higher latency for its additional F1 on these local hosts: Qwen3-8B's p50 increases from 122.5 to 271.4 ms relative to Always-Raw. Lite obtains a smaller F1 gain with p50 near Always-Raw. On the measured DeepSeek-V3.2 API example, zero-shot Lite obtains 63.1 F1, 0.182k tokens, and about 2.2 s end-to-end latency versus Always-Raw's 63.4 F1, 0.287k tokens, and about 2.1 s. This is a token-saving trade-off with slightly lower F1 and higher latency, not a three-axis Pareto improvement. Fresh Online Full FORGE latency was not measured on API hosts.

\paragraph{Cost-accounting note.}
Cached-result tables count selected-answer tokens and exclude earlier feature-acquisition calls, so they do not represent the end-to-end cost of a new query. Fresh Online evaluation also counts Full FORGE's greedy and self-consistency probes in tokens and latency; Lite has no host-dependent calls. Summary uses no auxiliary LLM call (Section~\ref{appendix:retrieval}), and its prompt tokens are included in both cost-accounting protocols.

\subsection{Reproducibility}
\label{appendix:repro}

Seed 42 identifies the canonical data split used for the point-estimate tables. The uncertainty analysis on Qwen3-8B and Mistral-7B instead evaluates three independent train/dev/test splits with three router-training seeds per split (nine runs per method and host). Each tunable baseline is retuned on its corresponding dev split with the same search budget. Nine-run standard deviations summarize split and initialization variation; paired-query bootstrap confidence intervals compare methods on the canonical test split and are not nine-run intervals.
The bootstrap resampling unit is the matched test query.
The \emph{Oracle (6-arm)} reference rows are enumerated references that use the observed outcomes of all six arms for each query. They are not deployable policies and are reported as descriptive references rather than F1 or EM upper bounds.
Exact reproduction requires model outputs, Stage-0 enumeration tables, Stage-1 warm-start checkpoints, and Stage-2 GRPO checkpoints in addition to the method settings reported here. Given pre-enumerated arm outputs, the reported Stage-1 supervised pipeline completes on one CPU in under two minutes. Stage-2 reproduction additionally requires the host endpoints used for rollout.

\paragraph{Hyperparameter selection.}
The Pareto utility's cost weights $(\lambda_\mathrm{in}, \lambda_\mathrm{out}) = (0.10, 0.20)$ and the Boltzmann temperature $\tau = 1.0$ are selected once on the Qwen3-8B HotpotQA dev set ($N{=}500$) and held fixed across the original host and benchmark comparisons. These weights specify a study operating point, not a universal price for user value. Table~\ref{tab:lambda} reports a canonical cached sensitivity grid; its F1 range is 1.2 points on the non-thinking host and 2.5 points on the thinking-capable host.
Baselines with tunable scalars receive the same dev-budget tuning protocol on the same dev split: the BM25-Threshold rule's two thresholds are grid-searched, Adaptive-RAG's complexity-classifier threshold is calibrated, and the Sysformer adapter rank is selected from $\{4, 8, 16\}$.
In cross-host transfer (Section~\ref{sec:transfer}), $\lambda$ and $\tau$ are not retuned per target host and no target-host Stage 2 training is run. Zero-shot describes the transferred weights, not the absence of Full FORGE's target-host feature probes on a fresh query; the API transfer tables use cached-feature accounting.

\subsection{Supplementary Figures}
\label{app:supp_figures}

This subsection collects the schematic of the three extrinsic axes of adaptive inference (Figure~\ref{fig:axes}).

\begin{figure}[!htb]
\centering
\includegraphics[width=0.85\linewidth]{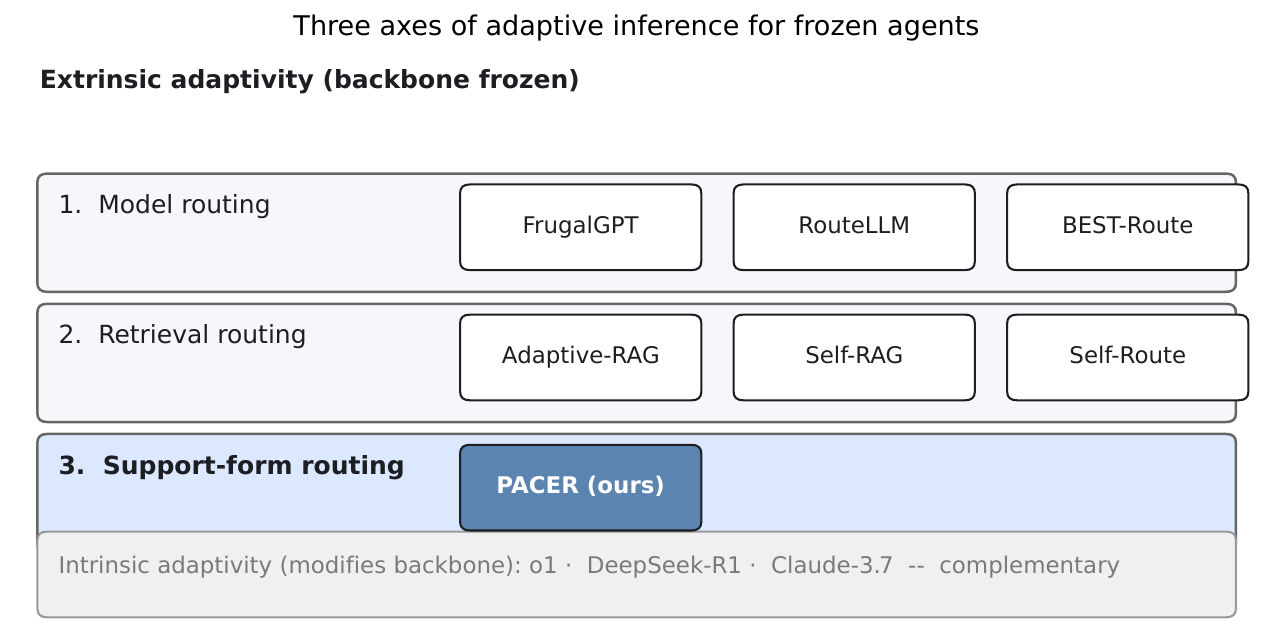}
\caption{\textbf{Three extrinsic axes of adaptive inference for frozen agents.} Model routing~\citep{chen2024frugalgpt, ding2025bestroute} and retrieval routing~\citep{jeong2024adaptiverag, asai2024selfrag} have been studied at length. Support-form routing, the third axis, is the focus of this paper. An orthogonal intrinsic axis exists on hosts that expose a reasoning budget~\citep{openai2024o1, deepseek2025r1}, and our framework composes with it where available.}
\label{fig:axes}
\end{figure}

\section{Additional Ablations}
\label{app:additional}

This appendix collects ablations referenced from the main text but deferred for space.

\subsection{Stage-2 Training and KL-Anchor Ablations}
\label{app:training_method}

Table~\ref{tab:training_ablation} puts the Stage-2 algorithm comparison and KL-anchor sweep on the same two benchmarks and cost scale. The shared three-arm Stage-1 row is a pipeline reference rather than a same-alphabet GRPO control; Panel A compares six-arm refinements, and Panel B varies the GRPO anchor around the default row in Panel A. The matched six-arm, nine-run comparison reported above isolates Stage 2 from alphabet expansion.
GRPO and Dr.\,GRPO differ by at most $0.2$ F1 in the displayed point estimates on the two benchmarks. Both exceed DPO and RLOO here, but this table alone does not establish a statistical tie or measure the matched six-action Stage-2 gain across all five tasks. That controlled comparison gives $+1.8$ and $+1.7$ macro-F1 on Qwen3-8B and Mistral-7B, respectively, at 1.28M model completions per operating point (Appendix~\ref{appendix:compute}).
GRPO is the reported primary configuration; the comparisons among DPO, RLOO, Dr.\,GRPO, and GRPO apply to this implementation and tested action alphabets.

\begin{table}[!htb]
\small
\renewcommand{\arraystretch}{1.06}
\caption{\textbf{Stage-2 refinement and KL-anchor ablations on Qwen3-8B.} Canonical Cached F1 (\%) and selected-answer token cost $\bar C$ (k) on HotpotQA and MuSiQue. Panel A compares six-arm refinement methods from the same Stage-1 checkpoint. Panel B varies $\beta$ for six-arm GRPO; the highlighted GRPO row in A is its default adaptive $\beta{=}0.05$ point. The three-arm Stage-1 row is shown once as a shared pipeline reference, not a matched six-arm $\beta{=}\infty$ limit.}
\label{tab:training_ablation}\label{tab:kl_ablation_app}
\centering
\setlength{\tabcolsep}{6pt}
\forgeTableFit{0.95\columnwidth}{%
\begin{tabular}{l cc cc c}
\toprule
\multirow{2}{*}{\textbf{Configuration}}
& \multicolumn{2}{c}{\texttt{HotpotQA}}
& \multicolumn{2}{c}{\texttt{MuSiQue}}
& \multirow{2}{*}{\textbf{Avg F1}\,$\uparrow$} \\
\cmidrule(lr){2-3}\cmidrule(lr){4-5}
 & F1\,$\uparrow$ & $\bar{C}$\,$\downarrow$ & F1\,$\uparrow$ & $\bar{C}$\,$\downarrow$ & \\
\midrule
\multicolumn{6}{l}{\textit{Shared three-arm pipeline reference}} \\
Offline KL only (FORGE (Stage 1))                                & 58.3          & 0.272 & 20.8          & 0.295 & 39.6 \\
\midrule
\multicolumn{6}{l}{\textit{A. Stage-2 algorithm (six-arm alphabet)}} \\
\textcolor{natureteal}{DPO}~\citep{rafailov2023dpo}              & 59.5          & 0.258 & 25.0          & 0.265 & 42.3 \\
\textcolor{natureteal}{RLOO}~\citep{ahmadian2024rloo}            & 59.2          & 0.263 & 24.6          & 0.272 & 41.9 \\
\textcolor{natureteal}{Dr.\,GRPO}~\citep{liu2025drgrpo}          & 60.6          & 0.235 & 26.5          & 0.249 & 43.6 \\
\rowcolor{lightblue!60} \textcolor{naturemagenta}{\textbf{GRPO ($\beta{=}0.05$, default)}} & \textbf{60.4} & \textbf{0.236} & \textbf{26.4} & \textbf{0.252} & \textbf{\fbox{43.4}} \\
\midrule
\multicolumn{6}{l}{\textit{B. GRPO KL-anchor coefficient $\beta$ (six-arm alphabet)}} \\
$\beta = 0$ (no anchor)                    & 56.2 & 0.412 & 22.1 & 0.398 & 39.2 \\
$\beta = 0.01$ (weak)                      & 59.6 & 0.252 & 25.5 & 0.265 & 42.6 \\
$\beta = 0.20$ (strong)                    & 59.1 & 0.265 & 24.7 & 0.282 & 41.9 \\
\bottomrule
\end{tabular}%
}
\end{table}

\subsection{Factorized versus Flat Policy Head}
\label{app:factorization}

Table~\ref{tab:factorization} compares the default factorized head $\pi_\theta(a \mid x) = \pi_\theta^\mu(\mu \mid h(x)) \cdot \pi_\theta^\theta(\theta \mid h(x), \mu)$ against a flat $K$-way categorical head over the joint action $a = (\mu, \theta)$, with feature vector and Stage-2 hyperparameters held fixed.
At $K{=}6$ the displayed factorized-versus-flat difference is $0.6$ F1; at $K{=}12$ it is $1.6$ F1, with the reported parameter counts 269K versus 273K. These two tested settings are consistent with a benefit from conditioning thinking depth on support form, but do not establish a general scaling law in $K$.
All main-paper results use the factorized head.

\begin{table}[!htb]
\small
\renewcommand{\arraystretch}{1.06}
\caption{\textbf{Factorized vs flat policy head.} FORGE F1 (\%), average cost $\bar C$ (k), and policy parameter count under two architectures: a flat $K$-way categorical head over the joint action $a = (\mu, \theta)$, and the default factorized head $\pi^\mu(\mu \mid h(x)) \cdot \pi^\theta(\theta \mid h(x), \mu)$ (Section~\ref{sec:router}). Reported on \texttt{Qwen3-8B} ($K{=}6$) and \texttt{Qwen3-8B-Thinking} ($K{=}12$).}
\centering
\setlength{\tabcolsep}{6pt}
\forgeTableFit{0.95\columnwidth}{%
\begin{tabular}{l cc cc c}
\toprule
\multirow{2}{*}{\textbf{Policy head}}
& \multicolumn{2}{c}{\texttt{Qwen3-8B} ($K{=}6$)}
& \multicolumn{2}{c}{\texttt{Qwen3-8B-Thinking} ($K{=}12$)}
& \multirow{2}{*}{\textbf{\# params}} \\
\cmidrule(lr){2-3}\cmidrule(lr){4-5}
 & F1\,$\uparrow$ & $\bar{C}$\,$\downarrow$ & F1\,$\uparrow$ & $\bar{C}$\,$\downarrow$ & \\
\midrule
Flat $K$-way categorical                       & 59.8 & 0.241 & 63.2 & 0.715 & 273K \\
\rowcolor{lightblue!60} Factorized $\pi^\mu \cdot \pi^{\theta \mid \mu}$ (default) & \textbf{60.4} & \textbf{0.236} & \textbf{64.8} & \textbf{0.690} & \textbf{269K} \\
\bottomrule
\end{tabular}%
}
\label{tab:factorization}
\end{table}

\begin{table}[!htb]
\small
\renewcommand{\arraystretch}{1.06}
\caption{\textbf{Canonical cached sensitivity to cost weights $\lambda_\mathrm{in}$ and $\lambda_\mathrm{out}$.} FORGE F1 (\%) and selected-answer token cost $\bar C$ (k) on \texttt{HotpotQA} for non-thinking Qwen3-8B and thinking-capable Qwen3-8B-Thinking. Across the surveyed points, F1 spans 1.2 and 2.5 points, respectively. The study default $(0.1, 0.2)$ is highlighted; these weights are not universal prices.}
\label{tab:lambda}
\centering
\setlength{\tabcolsep}{6pt}
\forgeTableFit{0.8\columnwidth}{%
\begin{tabular}{cc cc cc}
\toprule
\multirow{2}{*}{$\lambda_\mathrm{in}$} & \multirow{2}{*}{$\lambda_\mathrm{out}$}
& \multicolumn{2}{c}{Non-thinking host}
& \multicolumn{2}{c}{Thinking-capable host} \\
\cmidrule(lr){3-4}\cmidrule(lr){5-6}
& & F1\,$\uparrow$ & $\bar{C}$\,$\downarrow$ & F1\,$\uparrow$ & $\bar{C}$\,$\downarrow$ \\
\midrule
0.05 & 0.0  & 60.8 & 0.262 & 65.5 & 0.852 \\
0.05 & 0.2  & 60.7 & 0.246 & 65.0 & 0.731 \\
0.10 & 0.0  & 60.4 & 0.253 & 65.2 & 0.802 \\
\rowcolor{lightblue!60} 0.10 & 0.20 (default) & \textbf{60.4} & \textbf{0.236} & \textbf{64.8} & \textbf{0.690} \\
0.10 & 0.50 & 59.9 & 0.218 & 63.5 & 0.582 \\
0.20 & 0.20 & 59.6 & 0.210 & 63.0 & 0.624 \\
\bottomrule
\end{tabular}%
}
\end{table}

\noindent Table~\ref{tab:lambda} shows a quality--cost trade-off rather than invariance to the cost weights: the non-thinking host spans 1.2 F1 points and the thinking-capable host spans 2.5 points across the displayed grid. The table uses cached selected-answer tokens; the Fresh Online token curves also include feature acquisition.
Increasing $\lambda_\mathrm{in}$ shifts the policy toward cheaper input arms (Direct over Raw), and increasing $\lambda_\mathrm{out}$ specifically suppresses thinking-budget arms on thinking-capable hosts.
The default of $(0.1, 0.2)$ was selected on a held-out development set and held fixed for all main results.

\begin{table}[H]
\small
\renewcommand{\arraystretch}{1.06}
\caption{\textbf{Training-size sweep.} HotpotQA F1 (\%) for Stage 1 and Stage 1+2 on Qwen3-8B.}
\label{tab:data_size}
\centering
\small
\setlength{\tabcolsep}{10pt}
\begin{tabular}{@{}ccc@{}}
\toprule
$N_\mathrm{warm}$ & FORGE (Stage 1) & FORGE (Stage 1+2) \\
\midrule
200    & 46.8 & 50.6 \\
500    & \textbf{58.5} & \textbf{60.2} \\
1{,}000  & 58.2 & 60.4 \\
\rowcolor{lightblue!60} 1{,}500 (default) & \textbf{58.3} & \textbf{60.4} \\
2{,}000  & 58.4 & 60.5 \\
\bottomrule
\end{tabular}
\end{table}

\noindent Table~\ref{tab:data_size} shows empirical Stage-1 saturation around $N_\mathrm{warm} = 500$ on this HotpotQA sweep; no finite-sample theorem here predicts that threshold.
The Stage 1-to-Stage 1+2 difference ranges from $+1.7$ to $+3.8$ F1 across the displayed training sizes, and is $+2.1$ at the default $N_\mathrm{warm} = 1{,}500$. This single-benchmark sweep does not isolate GRPO from action-alphabet changes in the main three-arm-to-six-arm comparison.
Stage 2 needs rollout calls but no extra warm-start data.

\subsection{Per-Arm Allocation Changes from Three to Six Actions}
\label{app:perarm_breakdown}

The original three-arm Stage-1 to six-arm final-policy comparison differs in both action availability and optimization; its $+3.5$ Qwen3-8B and $+3.6$ Mistral-7B macro-F1 differences are not GRPO-only effects. The matched six-action, nine-run comparison isolates the Stage-2 increment at $+1.8$ [1.3, 2.3] and $+1.7$ [1.2, 2.2] macro-F1, respectively (paired 95\% confidence intervals on the canonical test split).
Table~\ref{tab:perarm_breakdown} describes the canonical three-arm-to-six-arm allocation changes on Qwen3-8B. Its per-arm terms are descriptive accounting terms, computed as $\Delta\mathrm{alloc}_a \cdot (m_a^\mathrm{F1} - m_\mathrm{ref}^\mathrm{F1})$ for the selected query subsets; they should not be interpreted as controlled causal effects of GRPO.
Here CoT-Prompt aggregates the three CoT-Prompt arms across support forms.

\begin{table}[!htb]
\small
\renewcommand{\arraystretch}{1.06}
\caption{\textbf{Descriptive per-arm accounting for Qwen3-8B's three-arm Stage-1 to six-arm final-policy comparison.} Allocations sum to $100\%$ within each stage, and the displayed terms sum (up to rounding) to the observed $+3.5$ macro-F1 difference. The action alphabet also changes, so this is not the controlled GRPO-only estimate; Mistral-7B's aggregate difference is $+3.6$.}
\label{tab:perarm_breakdown}
\centering
\small
\setlength{\tabcolsep}{6pt}
\renewcommand{\arraystretch}{0.95}
\forgeTableFit{0.95\columnwidth}{%
\begin{tabular}{lcccc}
\toprule
Arm $a$ & Stage-1 alloc.\ (\%) & Final alloc.\ (\%) & $\Delta$ alloc.\ (pp) & Descriptive F1 term \\
\midrule
Direct           & 13 & 18 & $+5$  & $+0.4$ \\
Summary          & 27 & 32 & $+5$  & $+0.9$ \\
Raw              & 60 & 35 & $-25$ & $-0.7$ \\
CoT-Prompt       &  0 & 15 & $+15$ & $+2.9$ \\
\midrule
\textbf{Total}   & 100 & 100 & --- & \textbf{+3.5} \\
\bottomrule
\end{tabular}%
}
\end{table}

\noindent The largest descriptive term is the $+2.9$ associated with CoT-Prompt, which the three-arm warm-start policy cannot select. Direct and Summary shifts contribute $+1.3$ combined, while the Raw term is $-0.7$ in this accounting. These terms describe the changed allocations and query subsets; they do not partition the effect of GRPO from that of adding CoT-Prompt. The matched six-arm comparison above provides that controlled Stage-2 estimate.

\subsection{Router Behavior in BGE Embedding Space}
\label{app:router_umap}

This subsection describes routing in BGE embedding space (Figure~\ref{fig:router_umap}), and against the per-instance utility oracle (Figure~\ref{fig:router_oracle_vs_forge}). These plots show action allocation without identifying the cause of each answer. In these figures, CoT-Prompt aggregates the CoT-Prompt arms across support forms, and Direct, Summary, and Raw denote the corresponding NoThink arms.

\paragraph{Per-point structure in BGE embedding space.}
Figure~\ref{fig:router_umap} plots a 2D UMAP projection of the BGE-base sentence embeddings of the same $N{=}1{,}500$ held-out test queries, colouring each point by FORGE's argmax action.
Five visible cluster regions correspond to five query types derived from dataset-provided labels (\emph{Comparison}, \emph{Factoid}, \emph{Multi-hop 2--3 hops}, \emph{Multi-hop 4+ hops}, \emph{Verification}); they emerge from the BGE embedding alone, before the router sees them.
Most clusters are dominated by a single arm, and mixed colours concentrate near cluster boundaries.
The regions are consistent with query-dependent routing in the displayed BGE projection, but the visualization alone cannot rule out memorization or establish why an individual action succeeds.

\begin{figure}[!htb]
\centering
\includegraphics[width=0.85\linewidth]{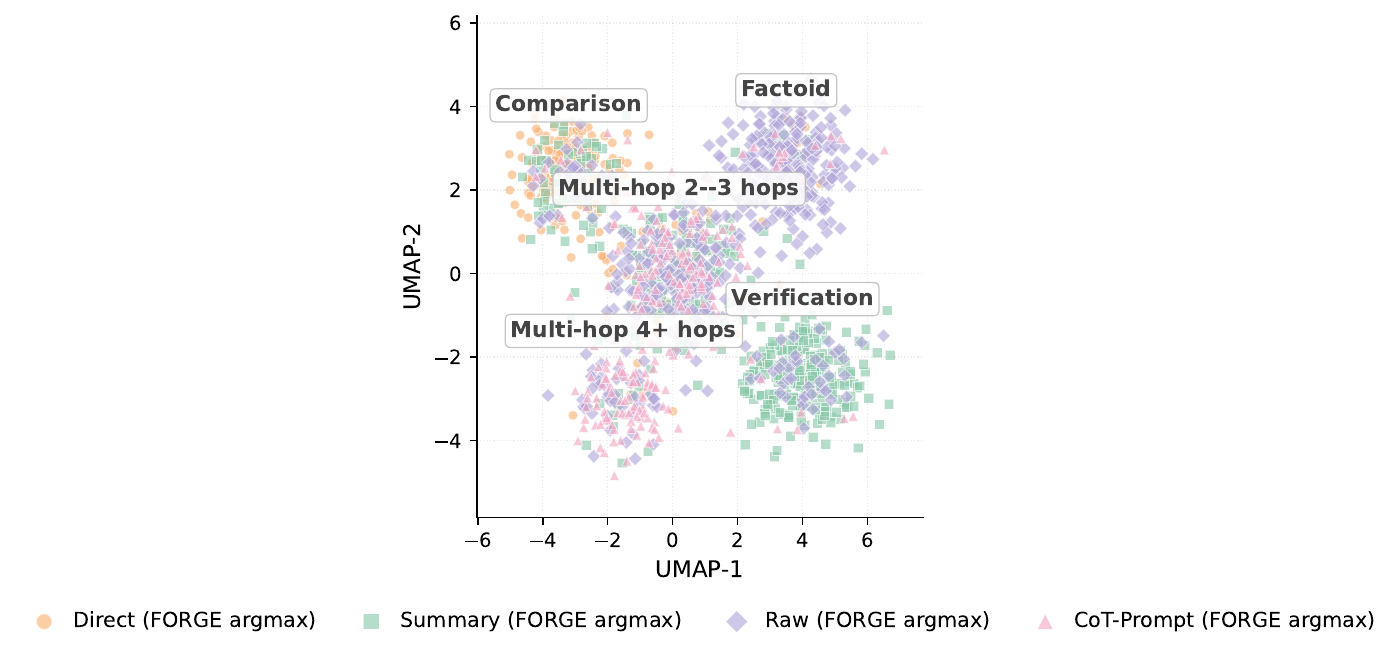}
\caption{\textbf{2D UMAP projection of held-out BGE query embeddings ($N{=}1{,}500$), coloured by FORGE's argmax action on Qwen3-8B.} The five visible clusters correspond to five query types derived from dataset labels (annotated near each cluster centroid). UMAP uses frozen BGE-base query embeddings, $n_\mathrm{neighbors}{=}30$, $\mathrm{min\_dist}{=}0.3$, and cosine distance; centroid labels are illustrative.}
\label{fig:router_umap}
\end{figure}

\paragraph{Alignment with per-instance oracle.}
Figure~\ref{fig:router_umap} shows that FORGE's routing is structured in embedding space but not yet that the structure is \emph{correct}.
Figure~\ref{fig:router_oracle_vs_forge} places FORGE's routing beside the per-instance \emph{utility oracle} on the same UMAP projection. This diagnostic oracle evaluates all six actions and selects the one maximizing the stated F1-minus-token-cost utility; it is separate from the three-arm Stage-0 warm-start enumeration and from the tabular Oracle reference rows. FORGE selects its answer action without enumerating all arms, although Full FORGE still needs host-dependent feature probes on a fresh query.
The two panels are nearly indistinguishable: \textbf{FORGE matches the per-instance oracle on $87\%$ of queries}, with the residual $13\%$ disagreement concentrated at the boundaries between clusters where multiple arms are near-tied in Pareto utility (visible as the small fraction of mismatched-colour points within each cluster, especially on the multi-hop $2$--$3$-hop region where Raw and CoT-Prompt are near-tied for many queries).
The $87\%$ agreement is with the specified utility oracle; the UMAP plot does not explain the outcome of individual answers.

\begin{figure}[!htb]
\centering
\includegraphics[width=0.95\linewidth]{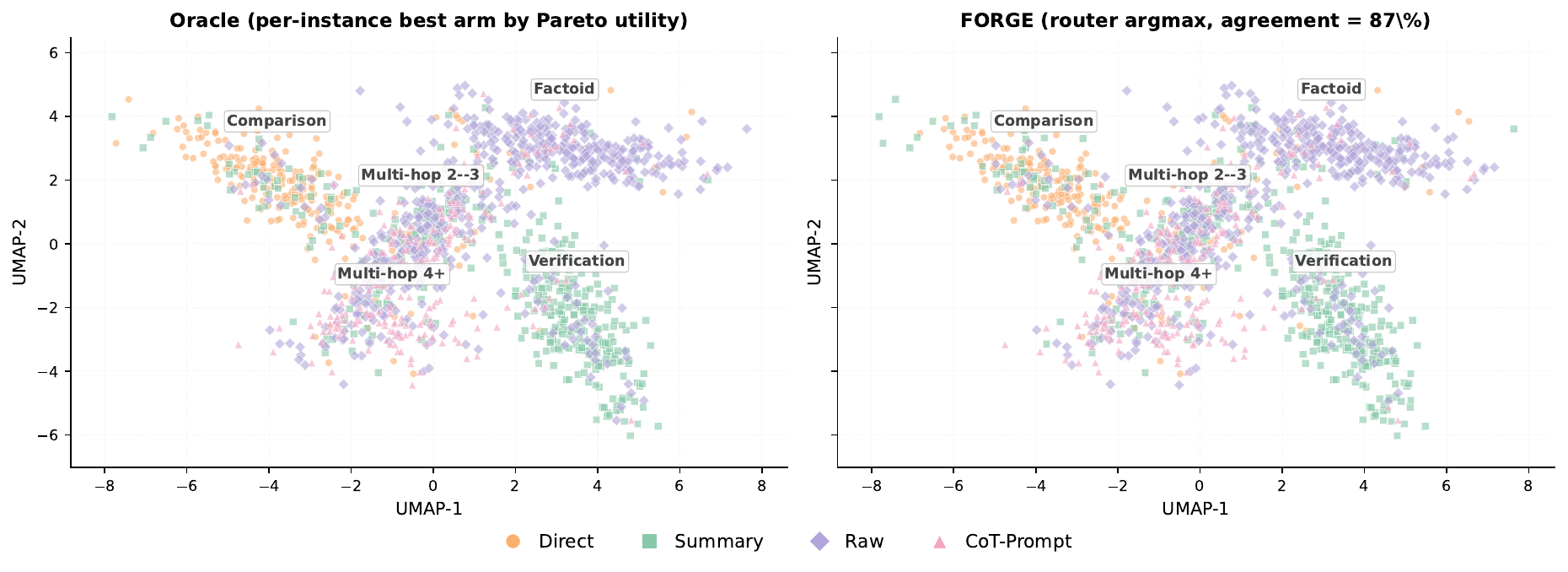}
\caption{\textbf{FORGE's argmax action agrees with the six-arm utility oracle on $87\%$ of the same queries ($N{=}1{,}500$, Qwen3-8B).} \emph{Left:} the arm maximizing the stated F1-minus-token-cost utility after offline six-arm evaluation; this diagnostic is separate from the tabular Oracle reference rows and from three-arm Stage 0. \emph{Right:} the trained router's argmax without answer-arm enumeration. Coordinates and cluster centroids are shared. The residual $13\%$ action disagreement is concentrated near the displayed multi-hop cluster boundary; this plot does not measure Fresh Online probe cost.}
\label{fig:router_oracle_vs_forge}
\end{figure}

\subsection{Cross-Host Transfer: Per-Task Breakdown}
\label{app:nim_transfer_full}

Table~\ref{tab:nim_transfer} reports macro F1 for cross-host transfer; Table~\ref{tab:nim_transfer_full} adds per-task F1, Always-Direct (the parametric-only reference), and FORGE Stage 1 (the supervised warm-start ablation).

\begin{table}[!htbp]
\small
\renewcommand{\arraystretch}{1.06}
\centering
\caption{\textbf{Zero-shot cross-host transfer: per-task F1 and cached cost.} A source-trained FORGE router is applied to three API hosts without target-host weight updates.}
\label{tab:nim_transfer_full}
\begin{threeparttable}
\small
\setlength{\tabcolsep}{3pt}
\renewcommand{\arraystretch}{1.03}
\begin{tabular}{@{}llccccc c c c@{}}
\toprule
\multirow{2}{*}{Target host} & \multirow{2}{*}{Policy}
  & \multicolumn{5}{c}{F1 by task (\%)}
  & \multirow{2}{*}{\textbf{Avg F1}\,$\uparrow$}
  & \multirow{2}{*}{\makecell{$\bar C$ (k)\,$\downarrow$\\saving}}
  & \multirow{2}{*}{$\Delta_{\mathrm{Raw}}^{F_1}$} \\
\cmidrule(lr){3-7}
 & & HQA & 2Wiki & MuSiQue & PopQA & FEVER & & & \\
\midrule
\multirow{6}{*}{\makecell[l]{\texttt{Llama-3.3}\\\texttt{70B}}}
 & Direct  & 42.9 & 32.4 & 13.4 & 38.5 & 61.3 & 37.7 & \makecell{0.085\\\textcolor{natureteal}{($-$73\%)}} & -20.8 \\
 & Summary & 57.6 & 37.1 & 20.5 & 90.3 & 74.7 & 56.0 & \makecell{0.214\\\textcolor{natureteal}{($-$33\%)}} & -2.5 \\
 & Raw     & \textbf{64.7} & 41.5 & 26.9 & 88.8 & 70.7 & 58.5 & \makecell{0.318\\\textcolor{gray}{(ref.)}} & +0.0 \\
\cmidrule(lr){2-10}
 & \emph{Oracle} & \emph{73.8} & \emph{54.7} & \emph{36.3} & \emph{90.8} & \emph{76.0} & \emph{66.3} & \makecell{\emph{0.145}\\\textcolor{natureteal}{($-$54\%)}} & --- \\
\cmidrule(lr){2-10}
 & Stage 1 & 61.9 & 45.3 & 24.1 & 90.3 & 72.0 & 58.7\,$\pm$\,0.4 & \makecell{0.215\\\textcolor{natureteal}{($-$32\%)}} & +0.2 \\
\rowcolor{lightblue!60}\cellcolor{white} & \textbf{FORGE} & 64.0 & \textbf{47.8} & \textbf{27.2} & \textbf{91.5} & \textbf{75.5} & \textbf{\fbox{61.2\,$\pm$\,0.6}} & \makecell{0.205\\\textcolor{natureteal}{($-$36\%)}} & +2.7 \\
\midrule
\multirow{6}{*}{\makecell[l]{\texttt{DeepSeek}\\\texttt{V3.2}}}
 & Direct  & 48.2 & 29.2 & 22.5 & 39.6 & 49.3 & 37.8 & \makecell{0.055\\\textcolor{natureteal}{($-$81\%)}} & -25.6 \\
 & Summary & 58.4 & 43.3 & 28.3 & \textbf{94.8} & \textbf{88.0} & 62.6 & \makecell{0.184\\\textcolor{natureteal}{($-$36\%)}} & -0.8 \\
 & Raw     & \textbf{65.6} & 42.8 & 31.8 & 91.3 & 85.3 & \textbf{63.4} & \makecell{0.287\\\textcolor{gray}{(ref.)}} & +0.0 \\
\cmidrule(lr){2-10}
 & \emph{Oracle} & \emph{72.2} & \emph{61.0} & \emph{44.4} & \emph{96.1} & \emph{96.0} & \emph{73.9} & \makecell{\emph{0.116}\\\textcolor{natureteal}{($-$60\%)}} & --- \\
\cmidrule(lr){2-10}
 & Stage 1 & 63.3 & 42.8 & 30.7 & 92.9 & 72.0 & 60.3\,$\pm$\,0.4 & \makecell{0.183\\\textcolor{natureteal}{($-$36\%)}} & -3.1 \\
\rowcolor{lightblue!60}\cellcolor{white} & \textbf{FORGE} & 65.4 & \textbf{45.5} & \textbf{33.5} & 94.0 & 78.5 & \textbf{\fbox{63.4\,$\pm$\,0.7}} & \makecell{0.176\\\textcolor{natureteal}{($-$39\%)}} & +0.0 \\
\midrule
\multirow{6}{*}{\makecell[l]{\texttt{Qwen3.5}\\\texttt{397B}}}
 & Direct  & 54.0 & 25.0 & 19.1 & 38.9 & 66.7 & 40.7 & \makecell{0.062\\\textcolor{natureteal}{($-$80\%)}} & -22.4 \\
 & Summary & 58.7 & 44.8 & 23.4 & \textbf{95.6} & \textbf{88.0} & 62.1 & \makecell{0.196\\\textcolor{natureteal}{($-$36\%)}} & -1.0 \\
 & Raw     & 65.0 & 41.0 & 29.8 & 93.0 & 86.7 & 63.1 & \makecell{0.305\\\textcolor{gray}{(ref.)}} & +0.0 \\
\cmidrule(lr){2-10}
 & \emph{Oracle} & \emph{74.2} & \emph{56.0} & \emph{40.2} & \emph{95.6} & \emph{93.3} & \emph{71.9} & \makecell{\emph{0.123}\\\textcolor{natureteal}{($-$60\%)}} & --- \\
\cmidrule(lr){2-10}
 & Stage 1 & 63.1 & 46.5 & 30.7 & 94.5 & 82.7 & 63.5\,$\pm$\,0.4 & \makecell{0.209\\\textcolor{natureteal}{($-$31\%)}} & +0.4 \\
\rowcolor{lightblue!60}\cellcolor{white} & \textbf{FORGE} & \textbf{65.5} & \textbf{49.0} & \textbf{33.0} & 94.5 & 86.0 & \textbf{\fbox{65.6\,$\pm$\,0.5}} & \makecell{0.198\\\textcolor{natureteal}{($-$35\%)}} & +2.5 \\
\bottomrule
\end{tabular}
\begin{tablenotes}[flushleft]
\footnotesize
\item HQA abbreviates HotpotQA. Direct, Summary, and Raw are fixed policies. Each task has 75 canonical test queries. $\bar C$ counts selected-answer tokens (k/query); its second line is the reduction relative to Raw. $\Delta_{\mathrm{Raw}}^{F_1}$ compares macro F1 with Raw. Stage 1 and FORGE macro F1 are mean $\pm$ SD over three seeds on the fixed subset; per-task F1 and baseline scores are point estimates. The italic Oracle is an enumerated six-arm reference (Appendix~\ref{appendix:repro}) and is excluded from practical-policy rankings. Bold marks the highest practical per-task score; boxed macro F1 marks FORGE, which ties Raw on DeepSeek-V3.2 at displayed precision. The router uses 2{,}000 source training and 500 development queries. Pre-routing host calls are excluded from cached costs.
\end{tablenotes}
\end{threeparttable}
\end{table}

\subsection{Exact Match (EM) Scores}
\label{app:em_main}

The main paper reports token-level F1 throughout for direct comparability with prior routing work.
Table~\ref{tab:em_main} reports EM on the same canonical cached predictions as Table~\ref{tab:main}. Its \emph{Oracle (6-arm)} row is the enumerated reference of Table~\ref{tab:main} scored under EM (Appendix~\ref{appendix:repro}); it is a descriptive reference, not an EM upper bound. The relative ordering is largely preserved, with FORGE leading the non-oracle policies in macro EM on both hosts.
EM/F1 ratios are roughly $0.78$ on HotpotQA, $0.82$ on 2Wiki, $0.65$ on MuSiQue (multi-hop tolerates partial overlap), $0.87$ on PopQA, and $0.96$ on FEVER (three-way classification).

\begin{table}[H]
\small
\renewcommand{\arraystretch}{1.06}
\centering
\caption{\textbf{Exact Match on the canonical cached predictions of Table~\ref{tab:main}.} Scores are percentages.}
\label{tab:em_main}
\begin{threeparttable}
\small
\setlength{\tabcolsep}{4pt}
\renewcommand{\arraystretch}{1.05}
\begin{tabular}{@{}llccccc c@{}}
\toprule
\multirow{2}{*}{Host} & \multirow{2}{*}{Policy}
 & \multicolumn{5}{c}{EM by task (\%)}
 & \multirow{2}{*}{\textbf{Avg EM}\,$\uparrow$} \\
\cmidrule(lr){3-7}
 & & HQA & 2Wiki & MuSiQue & PopQA & FEVER & \\
\midrule
\multirow{10}{*}{\texttt{Qwen3-8B}}
 & Direct       & 18.5 & 21.2 &  6.1 & 12.8 & 52.0 & 22.1 \\
 & Summary      & 36.5 & 31.0 & 12.2 & 75.5 & \textbf{78.0} & 46.6 \\
 & Raw          & 45.0 & 33.5 & 15.7 & 73.0 & 76.0 & 48.6 \\
\cmidrule(lr){2-8}
 & \emph{Oracle (6-arm)} & \emph{53.5} & \emph{43.8} & \emph{22.6} & \emph{76.5} & \emph{86.0} & \emph{56.5} \\
\cmidrule(lr){2-8}
 & Adaptive-RAG & 44.0 & 33.0 & 15.0 & 69.5 & 74.5 & 47.2 \\
 & TierMem-2arm & 37.0 & 32.5 & 13.5 & 71.0 & \textbf{78.0} & 46.4 \\
 & s3           & 43.0 & 32.0 & 14.0 & 68.0 & 72.0 & 45.8 \\
 & Sysformer    & 45.0 & 32.0 & 13.7 & 75.0 & 73.0 & 47.7 \\
 & Stage 1      & 45.5 & 31.5 & 13.5 & 76.0 & 72.5 & 47.8 \\
\rowcolor{lightblue!60}\cellcolor{white} & \textbf{FORGE} & \textbf{47.5} & \textbf{35.0} & \textbf{17.2} & \textbf{76.5} & 77.0 & \textbf{\fbox{50.6}} \\
\midrule
\multirow{10}{*}{\texttt{Mistral-7B}}
 & Direct       & 17.5 & 14.3 &  4.7 & 19.5 & 45.0 & 20.2 \\
 & Summary      & 29.5 & 19.8 &  7.0 & 69.5 & 66.5 & 38.5 \\
 & Raw          & 36.5 & 23.5 &  8.4 & 65.0 & 69.0 & 40.5 \\
\cmidrule(lr){2-8}
 & \emph{Oracle (6-arm)} & \emph{47.0} & \emph{35.0} & \emph{14.0} & \emph{72.5} & \emph{78.5} & \emph{49.4} \\
\cmidrule(lr){2-8}
 & Adaptive-RAG & 35.5 & 22.0 &  9.0 & 63.5 & 67.0 & 39.4 \\
 & TierMem-2arm & 31.5 & 23.0 &  7.7 & 61.0 & 67.0 & 38.0 \\
 & s3           & 35.5 & 21.5 &  8.3 & 61.5 & 64.0 & 38.2 \\
 & Sysformer    & 37.0 & 19.0 &  8.5 & 65.0 & 66.0 & 39.1 \\
 & Stage 1      & 37.5 & 19.0 &  8.7 & 66.5 & 67.0 & 39.7 \\
\rowcolor{lightblue!60}\cellcolor{white} & \textbf{FORGE} & \textbf{39.5} & \textbf{25.0} & \textbf{10.0} & \textbf{70.0} & \textbf{69.5} & \textbf{\fbox{42.8}} \\
\bottomrule
\end{tabular}
\begin{tablenotes}[flushleft]
\footnotesize
\item HQA abbreviates HotpotQA. Direct, Summary, and Raw are fixed policies. Adaptive-RAG~\citep{jeong2024adaptiverag}, TierMem-2arm~\citep{zhu2026tiermem}, s3~\citep{jiang2025s3}, and Sysformer~\citep{sysformer2025} are the comparison routers. The italic Oracle is the enumerated six-arm reference (Appendix~\ref{appendix:repro}) and is excluded from the practical-policy ranking. Bold marks the best practical per-task EM (including ties); boxed macro EM marks the best practical policy on each host.
\end{tablenotes}
\end{threeparttable}
\end{table}

\subsection{Numerical Table for the Action Alphabet Sweep}
\label{app:fig4_tables}

The KL-anchor values appear alongside the Stage-2 algorithm comparison in Table~\ref{tab:kl_ablation_app}, Panel B, with the default point in Panel A. Table~\ref{tab:action_size_app} reports the action alphabet sweep. Adding CoT-Prompt ($K{=}3 \to 6$) raises F1 while lowering cost, adding Think-Low ($K{=}6 \to 9$) gives modest gains ($+1.0$ and $+1.1$ F1), and unlocking Think-High ($K{=}9 \to 12$) gives the largest absolute improvement.

\begin{table}[!htb]
\small
\renewcommand{\arraystretch}{1.06}
\caption{\textbf{Action alphabet size ablation.} FORGE F1 (\%) and cached selected-answer tokens $\bar C$ (k) on \texttt{HotpotQA} and \texttt{MuSiQue} as the alphabet grows from $K{=}3$ to $K{=}12$ on \texttt{Qwen3-8B-Thinking}.}
\label{tab:action_size_app}
\centering
\setlength{\tabcolsep}{6pt}
\forgeTableFit{0.85\columnwidth}{%
\begin{tabular}{c l cc cc}
\toprule
\multirow{2}{*}{$K$} & \multirow{2}{*}{Alphabet}
& \multicolumn{2}{c}{\texttt{HotpotQA}}
& \multicolumn{2}{c}{\texttt{MuSiQue}} \\
\cmidrule(lr){3-4}\cmidrule(lr){5-6}
 & & F1\,$\uparrow$ & $\bar{C}$\,$\downarrow$ & F1\,$\uparrow$ & $\bar{C}$\,$\downarrow$ \\
\midrule
3  & $\{$Direct, Summary, Raw$\} \times \{$NoThink$\}$  & 58.5 & 0.272 & 21.3 & 0.295 \\
6  & $\mathcal{A}_\mathrm{nt}$ (default non-thinking)    & 60.4 & 0.236 & 26.4 & 0.252 \\
9  & $\mathcal{A}_\mathrm{nt} \cup \mathrm{Think\text{-}Low}$ arms  & 61.4 & 0.380 & 27.5 & 0.420 \\
\rowcolor{lightblue!60} 12 & $\mathcal{A}_\mathrm{t}$ (full 12-arm)  & \textbf{64.8} & 0.690 & \textbf{33.5} & 0.760 \\
\bottomrule
\end{tabular}%
}
\end{table}

\subsection{Heuristic Routing Baselines}
\label{app:heuristic_baselines}

Two heuristic routing baselines were excluded from the main results in Table~\ref{tab:main} for compactness.
\textbf{Self-Routing} uses the host's own confidence (a normalized log-probability of the Direct response) thresholded at the dev-tuned value to decide between Direct and Raw.
\textbf{FrugalGPT-Cascade}~\citep{chen2024frugalgpt} runs a Direct $\to$ Summary $\to$ Raw threshold cascade, advancing whenever the previous arm's confidence falls below a dev-tuned threshold.
Both favor inexpensive actions in the reported runs, but they do not have identical behavior: Self-Routing improves Mistral-7B macro F1 from 24.2 to 28.2 in Table~\ref{tab:heuristic_baselines}. Its initial Direct confidence call, and the earlier calls in the cascade, must be included in any Fresh Online cost or latency comparison.

\begin{table}[!htb]
\small
\renewcommand{\arraystretch}{1.06}
\caption{\textbf{Heuristic routing baselines on the main hosts.} F1 (\%) and cached selected-answer token cost $\bar C$ (k) on the five benchmarks of Table~\ref{tab:main}. Cached $\bar C$ omits preliminary Direct/confidence and cascade calls; thus this is not a Fresh Online comparison. Always-Direct is a reference row.}
\label{tab:heuristic_baselines}
\centering
\setlength{\tabcolsep}{4pt}
\forgeTableFit{0.95\columnwidth}{%
\begin{tabular}{ll cccccc cc}
\toprule
Host & Method & \texttt{HotpotQA} & \texttt{2Wiki} & \texttt{MuSiQue} & \texttt{PopQA} & \texttt{FEVER} & Avg F1 & $\bar C$ & $\Delta_{\mathrm{Raw}}^{F_1}$ \\
\midrule
\multirow{3}{*}{\texttt{Qwen3-8B}}
 & Always-Direct          & 26.2 & 25.7 & 10.2 & 15.0 & 54.4 & 26.3 & 0.040 & -30.8 \\
 & Self-Routing           & 27.0 & 27.6 & 11.1 & 18.2 & 56.1 & 28.0 & 0.058 & -29.1 \\
 & FrugalGPT-Cascade      & 26.2 & 25.7 & 10.2 & 16.5 & 55.0 & 26.7 & 0.042 & -30.4 \\
\midrule
\multirow{3}{*}{\texttt{Mistral-7B}}
 & Always-Direct          & 25.2 & 17.7 &  7.9 & 23.0 & 47.0 & 24.2 & 0.041 & -23.0 \\
 & Self-Routing           & 32.4 & 24.2 & 11.1 & 24.0 & 49.5 & 28.2 & 0.060 & -19.0 \\
 & FrugalGPT-Cascade      & 25.3 & 17.9 &  8.0 & 23.5 & 47.5 & 24.4 & 0.043 & -22.8 \\
\bottomrule
\end{tabular}%
}
\end{table}

\section{Properties of the Routing Objective and Warm Start}
\label{app:policy_properties}
\label{app:regret}

These results characterize the finite-action offline target and show how the
warm-start reference shapes it; they give no convergence or regret guarantee
for the implemented optimizer.

\subsection{Value of Per-Query Routing}
\label{app:background_props}

\begin{proposition}[Value of per-query routing]
\label{prop:routing}
For integrable utilities, define $U^*(x)=\max_{a\in\mathcal A}U(x,a)$ as the best achievable per-query utility. Then
\[
\mathbb E_x[U^*(x)]\ \geq\ \max_{a\in\mathcal A}\mathbb E_x[U(x,a)].
\]
It is strict if every action $a$ is suboptimal on a set of queries with positive probability.
\end{proposition}

\begin{proof}
For each $a$ and $x$, $U^*(x)-U(x,a)\geq0$. Taking expectations gives the
weak inequality. Under the stated positive-probability condition, each
nonnegative difference has strictly positive expectation. Since
$\mathcal A$ is finite, the inequality is strict against the best fixed action.
\end{proof}

\subsection{A Fixed-Reference Regularized Objective}
\label{app:regularized_policy}

The Stage~2 loss includes both an entropy bonus and a KL penalty to the
warm-start policy. Their effect can be characterized for an idealized,
unrestricted categorical policy with fixed coefficients.

\begin{proposition}[Idealized regularized policy]
\label{prop:regularized_policy}
Fix $x$, a full-support reference policy $\pi_0(\cdot\mid x)$, and coefficients
$\alpha,\beta\geq0$ with $\alpha+\beta>0$. The unique maximizer over
$\pi(\cdot\mid x)\in\Delta(\mathcal A)$ of
\[
\sum_a\pi(a\mid x)U(x,a)
-\beta D_{\mathrm{KL}}\!\big(\pi(\cdot\mid x)\,\|\,\pi_0(\cdot\mid x)\big)
+\alpha H\!\big(\pi(\cdot\mid x)\big)
\]
is
\begin{equation}
\label{eq:regularized_policy}
\pi^\dagger(a\mid x)
=
\frac{\pi_0(a\mid x)^{\beta/(\alpha+\beta)}
      \exp\!\big(U(x,a)/(\alpha+\beta)\big)}
{\sum_{a'}\pi_0(a'\mid x)^{\beta/(\alpha+\beta)}
      \exp\!\big(U(x,a')/(\alpha+\beta)\big)}.
\end{equation}
\end{proposition}

\begin{proof}
Expanding the KL divergence and entropy shows that the objective equals
$(\alpha+\beta)\log Z(x)
-(\alpha+\beta)D_{\mathrm{KL}}(\pi(\cdot\mid x)\,\|\,\pi^\dagger(\cdot\mid x))$,
where $Z(x)$ is the denominator in Eq.~\eqref{eq:regularized_policy}.
The KL divergence is nonnegative and vanishes only when the two policies coincide, proving a unique optimum.
\end{proof}

When $\beta=0$, Eq.~\eqref{eq:regularized_policy} reduces to
Eq.~\eqref{eq:boltzmann} with $\tau=\alpha$. A nonuniform reference generally
changes the solution when $\beta>0$. The proposition describes the exact
objective with fixed coefficients and known utilities. Stage~2 instead uses
sampled, group-standardized rewards, a clipped policy surrogate, a finite
router, and an adaptive KL coefficient; the proposition gives no convergence
or regret guarantee for the implemented Stage~2 optimization procedure.

\subsection{Warm-Start KL Decomposition}
\label{app:warmstart_proof}

Stage~1 fits the support-form marginal on the enumerated warm-start actions.
For the full action set, the thinking-depth head assigns equal initial probability to every thinking option.

\begin{proposition}[Warm-start KL decomposition]
\label{prop:warm_start_factorization}
Fix $x$. Let $\pi^*$ be the full-action distribution of
Eq.~\eqref{eq:boltzmann}, with support-form marginal $\pi^*_\mu$ and
conditional thinking policy $\pi^*_{\theta\mid\mu}$. Suppose the warm-start
policy factorizes as
$\pi_{\theta_0}(\mu,\theta\mid x)
=\pi_{\theta_0}^{\mu}(\mu\mid x)/|\Theta|$, with
$\pi_{\theta_0}^{\mu}(\mu\mid x)>0$ for all $\mu$. Then
\begin{equation}
\label{eq:kl_decomposition}
\begin{aligned}
D_{\mathrm{KL}}\!\big(\pi^*(\cdot\mid x)\,\|\,\pi_{\theta_0}(\cdot\mid x)\big)
={}&
D_{\mathrm{KL}}\!\big(\pi^*_\mu(\cdot\mid x)\,\|\,
                       \pi_{\theta_0}^{\mu}(\cdot\mid x)\big)\\
&+\log|\Theta|
-\mathbb E_{\mu\sim\pi^*_\mu(\cdot\mid x)}
 \big[H(\pi^*_{\theta\mid\mu}(\cdot\mid x,\mu))\big].
\end{aligned}
\end{equation}
The contribution from the conditional thinking policy lies in $[0,\log|\Theta|]$.
\end{proposition}

\begin{proof}
The KL chain rule splits the joint divergence into a divergence between
support-form marginals and the $\pi^*_\mu$-weighted conditional divergence.
For each $\mu$,
$D_{\mathrm{KL}}(\pi^*_{\theta\mid\mu}\,\|\,
\mathrm{Uniform}(\Theta))
=\log|\Theta|-H(\pi^*_{\theta\mid\mu})$.
Substitution gives Eq.~\eqref{eq:kl_decomposition}; the interval follows
from $0\leq H(\pi^*_{\theta\mid\mu})\leq\log|\Theta|$.
\end{proof}

This identity separates support-form marginal mismatch from mismatch due to a
uniform thinking head. It bounds only the latter; the marginal term can be
larger. The identity describes initialization mismatch and does not predict the empirical F1 improvement from Stage~2.

\section{Interpretation of the Cost--Quality Score}
\label{app:lens_derivations}

Equation~\eqref{eq:pareto_utility} is a weighted-sum scalarization of
measured F1 and normalized token costs. Its weights select an operating
point according to the user's relative valuation of these quantities;
performance away from the reported weight sweep is an empirical question.
Without the entropy term, maximizing expected utility over all categorical
policies chooses a utility-maximizing action for each query (with arbitrary
mixing among ties). The Shannon-entropy regularizer yields the soft labels in
Eq.~\eqref{eq:boltzmann}.

The measured token count in this formulation is an operational resource
cost. It is not the mutual-information rate in Shannon's rate--distortion
function, so the cost-quality scalarization alone does not constitute a
rate--distortion derivation of the Boltzmann policy.

\section{Limitations and Future Work}
\label{app:limitations_future}

\subsection{Limitations}

The scope of this paper is knowledge-intensive QA and fact verification, the task families on which the Direct/Summary/Raw alphabet is most directly meaningful; reasoning-only benchmarks (e.g., MATH, GSM8K, AIME) where retrieval is rarely useful would require a different action alphabet (e.g., chain-of-thought depth or tool calls) and are not evaluated here.
Stage~0 enumeration cost remains linear in $|\mathcal{A}_\mathrm{warm}|$, which constrains how rich the warm-start alphabet can be at training time; Stage~2 GRPO removes this constraint for the full alphabet $\mathcal{A}$ but introduces a host-call rollout budget of roughly $T_\mathrm{grpo} \cdot B \cdot G$ per Pareto operating point.
FORGE is trained against a fixed retrieval backend, so joint optimization of retrieval and routing is beyond the scope of this work.
In addition, the current router relies on compact query-level and retrieval-derived features as input; for document-understanding tasks where long-range context, layout structure, or cross-document interactions are central, and especially for multimodal settings involving image-text evidence, a context-aware or multimodal router may be needed to fully capture the support requirements.
On tasks where a single support form uniformly dominates, notably FEVER on non-thinking hosts where Summary~$\approx$~Raw for $96\%$ of queries, the general-purpose router does not improve on a task-specific fixed policy; on thinking-capable hosts this gap closes once the thinking-arm composition restores routing dynamic range.
Transfer accuracy is sensitive to the capability gap between source and target hosts: a source-trained router's support-form allocation can mismatch the target host's parametric knowledge, as observed on FEVER with DeepSeek-V3.2, where FORGE trails both fixed Summary and fixed Raw.
Stage~0 arm enumeration and Stage~2 training require labeled queries to compute the F1 reward, although inference with a trained router does not require labels. An LLM judge or verifier could supply proxy rewards for unlabeled adaptation, pending validation against true answer quality.

\subsection{Future Work}

Three directions follow from the current formulation.
A target-side calibration step that updates the last linear layer from a small labelled subset of the target host is a natural remedy for the source-target capability-gap limitation above.
Coupling support-form routing with retrieval or model routing would enlarge the discrete action alphabet; its training cost and empirical value remain to be tested.
Extending Stage~2 to a sequential cascade, in which the router decides whether to escalate from Direct to Summary to Raw based on intermediate host responses, would require accounting for the extra host calls and evaluating the resulting policy separately.
Replacing BM25 with dense (DPR or Contriever) or hybrid sparse-dense retrieval changes the upstream retriever, while $\mathcal{M}$ still describes what enters the host prompt. Comparing these backends could test the scope of FORGE's improvements in answer quality and token efficiency.

\subsection{Broader Impacts}
\label{app:broader_impacts}

FORGE changes which evidence and reasoning setting a frozen host receives, so its quality and cost effects depend on the task, host, and feature-acquisition protocol. The cached-feature comparisons measure routed-answer tokens after Full FORGE's host-dependent features have been collected. In the measured Fresh Online Qwen3-8B setting, Full FORGE improves F1 from $57.1$ to $59.5$ and reduces host tokens from $0.430$k to $0.384$k per query relative to Always-Raw, but its four pre-routing host calls raise median latency from $122.5$ to $271.4$ ms. FORGE-Lite avoids those calls and reaches $58.7$ F1 at $0.242$k tokens and $128.0$ ms. These measurements show a deployment trade-off; token counts alone do not establish energy savings or lower end-to-end cost for every provider.

The cost weights $(\lambda_\mathrm{in},\lambda_\mathrm{out})$ encode a chosen quality--cost preference. More aggressive cost weighting can select less evidence and may worsen answers or remove the grounding needed for attribution. A strong fixed form can also be preferable on a particular task: on FEVER with a non-thinking host, the Summary-pinned baseline slightly exceeds the general router's F1. In settings where citations or other grounding requirements matter, they should be measured directly and included as deployment criteria rather than assumed to follow from F1 or a change in cost weights. Although the host weights remain frozen, changing its inputs can change its errors; we have not evaluated whether routing amplifies or reduces bias, hallucination, or misuse risks outside the reported tasks.

\section{Use of Large Language Models}
\label{app:llm_use}

An AI assistant helped revise the manuscript's prose, organization, and internal consistency, including reviewing theoretical claims for assumptions and scope. The authors are responsible for verifying the technical claims, references, and empirical results in the submitted version.

\end{appendices}
\end{document}